\documentclass[twoside,11pt]{article}

\usepackage{jmlr2e}
\usepackage{multirow}

\usepackage{graphicx,amsmath,siunitx,booktabs}
\usepackage{yhmath}
\usepackage{mathdots}

\usepackage{float}
\usepackage{subcaption}
\usepackage{algorithm}
\usepackage{algpseudocode}

\usepackage{booktabs,subcaption,amsfonts,dcolumn}
\algrenewcommand\algorithmicrequire{\textbf{Input:}}
\algrenewcommand\algorithmicensure{\textbf{Output:}}
\newtheorem{defi}{Definition}[section]

\usepackage{lastpage}

\ShortHeadings{}{}
\firstpageno{1}

\begin{document}

\title{
Multidimensional Data Analysis Based on Block Convolutional Tensor Decomposition
}

\author{\name Mahdi Molavi \email  \\
    \addr Department of Computer Science, Tarbiat Modares, Tehran-Iran
    \AND
    \name Mansoor Rezghi \email Rezghi@modares.ac.ir \\
    \addr Department of Computer Science, Tarbiat Modares, Tehran-Iran
    \AND
    \name Tayyebeh Saeedi \\
    \addr Department of Computer Science, Tarbiat Modares, Tehran-Iran}

\editor{My editor}

\maketitle

\begin{abstract}
    Tensor decompositions are powerful tools for analyzing multi-dimensional data in their original format. Besides tensor decompositions like Tucker and CP, Tensor SVD (t-SVD)  which is based on the t-product of tensors is another extension of SVD to tensors that  recently developed and has found numerous applications in analyzing high dimensional data. This paper offers a new insight into the t-Product  and shows that this product is a block convolution of two tensors with periodic boundary conditions. Based on this viewpoint, we propose a new tensor-tensor product called the $\star_c{}\text{-Product}$ based on Block convolution with reflective boundary conditions. Using a tensor framework, this product can be easily extended to tensors of arbitrary order. Additionally, we introduce a tensor decomposition based on our $\star_c{}\text{-Product}$ for arbitrary order tensors. Compared to t-SVD, our new decomposition has lower complexity, and experiments show that it yields higher-quality results in applications such as classification and compression.
\end{abstract}


\section{Introduction}
In recent years, the popularity of devices such as smartphones, digital cameras, and traffic cameras has led to an abundance of image and video data. To effectively use machine learning algorithms on this type of data, we need non-vector tools for representation. While matrices are well-suited for displaying gray-scale images, tensors offer an extension to matrices that can handle multidimensional data such as color images and videos. Although traditional machine learning algorithms can be applied to vectorized data, this approach has significant drawbacks. First, vectorization can destroy structural relationships between various features, such as the spatial relationships between pixels in image data. Second, folding data into vectors can lead to high-dimensional data, resulting in overfitting and the curse of dimensionality. To address these issues, tensor-based, multi-linear methods have been developed that can work directly with multidimensional data. This approach has gained significant attention in recent years, and several tensor-based algorithms have been developed for well-known machine learning algorithms such as SVM, PCA, and LDA, including STM \cite{STM}, MPCA \cite{MPCA}, and MLDA \cite{MLDA}, respectively. Moreover, tensor methods have also been used in neural network layers and deep learning to reduce network parameters while maintaining network quality.

Matrix decomposition methods, such as Singular Value Decomposition (SVD), are commonly used in machine learning applications such as classification \cite{CSVD}, clustering \cite{ClSVD}, and dimension reduction methods \cite{DSVD}, and have proven to be effective \cite{Zaki}. Due to the advantages of SVD, there have been efforts to extend it to tensors. Over the past decade, various extensions of SVD for tensors have been proposed, such as CP and Tucker \cite{CPTuker}. These methods are widely used in various applications, including EEG classification \cite{EEG}, image processing \cite{im}, rs-fMRI classification \cite{rsfMRI}, tensor robust principal component analysis, and tensor completion \cite{cp1, cp2}. Recently, Kilmer and Martin proposed a new decomposition method called t-SVD that decomposes an order-3 tensor into three tensors using a new product between tensors called the tensor-tensor product (t-Product) \cite{tSVD}.

The t-SVD decomposition has numerous applications in various fields such as tensor completion \cite{tensorcompletion}, video recovery \cite{videorec}, dynamic MRI reconstruction \cite{dMRI}, color image denoising \cite{tsvdapp4}, and tensor robust principal analysis \cite{TRPCA}. However, while the t-SVD decomposition becomes SVD on an order-2 tensor (matrix), the tensors obtained from the decomposition are not guaranteed to be real in their work and some related works \cite{Kilmer2, n-dim-kilmer}. This issue has recently been resolved by Canyi Lu et al. in \cite{TRPCA}.

Moreover, in \cite{n-dim-kilmer}, the t-Product and hence the t-SVD have been generalized for n-dimensional data. However, their definition of the t-Product involves the terms fold and unfold, which transform tensors into structured circulant block matrices. Additionally, the extension of the t-Product defined in \cite{n-dim-kilmer} for high dimensions (n-order tensors) has a recursive form, leading to increased complexity in the definition. This complexity becomes even more challenging for higher dimensions, making the definitions of t-Product and t-SVD more difficult to comprehend. For instance, in the algorithm (2) presented in \cite{n-dim-kilmer}, numerous folding and unfolding operations are used in different modes, making the implementation complex and time-consuming.

In this paper, we sight at the t-Product from another viewpoint and show that the t-Product is actually a block convolution with a periodic boundary between two tensors. Then, this viewpoint  gives us an opportunity to use image processing literature in artificial boundary conditions to define a new tensor product with better properties. Different kinds of artificial boundaries are used in image restoration for modeling of blurring process \cite{zboundary}. Zero, periodic, reflective, and anti-reflective are known boundary conditions. It has been shown that the blurring matrix with periodic and reflective boundary (by symmetric mask) is diagonalizable by Fourier and Cosine transformation respectively \cite{rboundary, pboundary}. Also, the computational complexity and quality of the modeling by the reflective boundary case are better than the periodic boundary case. So here, we use reflective boundary conditions and introduce a new tensor product called $\star_c{}\text{-Product}$ that works on order-n tensors. Finally, we present a novel tensor decomposition based on this product named $\star_c{}\text{-SVD}$. Based on this decomposition, we apply classification and clustering techniques and we show our work has better performance compared to other state-of-the-art methods.

In recent years, structured tensors have been proposed to model convolution between high-dimensional tensors. For example, in \cite{pboundary}, circulant tensors with arbitrary order were introduced and it was demonstrated that these tensors can be used to restore 3D images with periodic boundaries in a specific case. Additionally, in \cite{rezghi-toplitz}, Toeplitz tensors with arbitrary dimensions were used for convolution modeling.
In this paper, we show that t-Product and $\star_c{}\text{-Product}$ can be presented as structured Circulant and Toeplitz+Hankel tensors. This leads to the generalization of the t-Product and $\star_c{}\text{-Product}$ for any dimension using structured circulant and Toeplitz plus Hankel tensors, respectively.
The advantages of our proposed method are its simplicity in generalizing the t-Product and $\star_c{}\text{-Product}$ to arbitrary dimensions using structured circulant Toeplitz plus Hankel tensors. This method also inherits the useful properties of these products, such as their equivalent SVD. Moreover, it is easier to implement and has a lower computational cost compared to previous methods.

The rest of this paper is organized as takes after:
Section \ref{sec2} contains notations that we use in the paper and related works. Section \ref{sec3} includes of the our viewpoint to t-Product. In section \ref{sec4}, we present our new product and corresponding decomposition. Section \ref{Exp} represents some experimental results in compression, clustering, and classification.

\section{Notations and related works}\label{sec2}
\subsection{Notations}
In this paper, we denote tensors by boldface calligraphic letters, \textit{e.g.,} $ \mathcal{A} $. We use boldface uppercase letters for matrices, \textit{e.g.,} $ A $. We use the Python notation $ \mathcal{A}[:,:,i] $ to denote the $ i $-th frontal slice and it is denoted compactly as $ A_i$, also we denote the $ i $-th column and $ i $-th row of a matrix as $ A[:,i] $ and $ A[i,:]$, respectively. Vectors are denoted by boldface lowercase letters, \textit{e.g.,} $ a $, and scalars are denoted by lowercase letters, \textit{e.g.,} $ a $.

We define the mode-$\text{n}$ product of a tensor and a matrix as follows, as defined in \cite{moden1, pboundary}:
\begin{align*}
  \boldsymbol{\mathcal{G}} := (M){n}.\boldsymbol{\mathcal{Z}},
\end{align*}
where $ M \in \mathbb{R}^{J \times I_n}$, is a matrix, $ \boldsymbol{\mathcal{Z}} \in \mathbb{R}^{I_1 \times I_2 \times \cdots \times I_{n} \times \cdots \times I_N}$ is an $ N $-order tensor, and $ \boldsymbol{\mathcal{G}} \in \mathbb{R}^{I_1 \times \cdots \times I_{n-1} \times J \times I_{n+1} \times \cdots \times I_N} $.

Really,  in mode-$\text{n}$ product each mode-$\text{n}$ fiber of tensor is multiplied by matrix, \cite{moden}.
This inner product and other matrix-vector, matrix-matrix, and also matrix-tensor products can be considered as a special case of contraction product of two tensors, which is a tensor product followed by a contraction along specified modes.
For two tensors $ \mathcal{A} \in \mathbb{R}^{I_1 \times \cdots \times I_D \times J_1 \times \cdots \times J_M} $ and $ \mathcal{B} \in \mathbb{R}^{K_1 \times \cdots \times K_L \times J_1 \times \cdots \times J_M} $, their contractive product corresponding to contraction modes $ D+1, \cdots, D+M $ and $ L+1, \cdots, L+M $ of tensors $ \mathcal{A} $ and $ \mathcal{B} $, respectively, can be defined as follows \cite{alg862}
\begin{align}
  \mathcal{C} = \langle \mathcal{A}, \mathcal{B} \rangle_{D+1, \cdots, D+M; L+1, \cdots, L+M} = \langle \mathcal{A}, \mathcal{B} \rangle_{D+1:D+M; L+1:L+M}
\end{align}
where
\begin{align*}
  \boldsymbol{c}_{i_1, \cdots, i_D, k_1, \cdots, k_L} = \sum_{1 \leq j_k \leq J_k} \boldsymbol{a}_{i_1, \cdots, i_D, j_1, \cdots, j_M} \boldsymbol{b}_{k_1, \cdots, k_L, j_1, \cdots, j_M}.
\end{align*}
The extension of diagonal concept to tensors is not unique, but the general from of such concept is proposed in \cite{pboundary}, which can cover the other definitions.
For arbitrary order $\mathcal{A}$, let $ S = \{s_1, \cdots, s_t\} $ be a subset of  modes $ \{1, \cdots, N\} $,  $ \mathcal{A}={\sf diag}(\mathcal{D}) \in \mathbb{R}^{I_1 \times \cdots \times I_N} $ is  $ \{S\} $-diagonal, if  $ A_{i_1 \cdots i_N} $ can be nonzero only if $ i_{s_1} = \cdots = i_{s_t} $ and are elements of $(N-|S|+1)$-order tensor $\mathcal{D}$.
For Example
we say that $3$-order tensor $ \mathcal{A} \in \mathbb{R}^{I_1 \times I_2 \times I_3} $ is $ \{\{ 1, 3 \} \} $-diagonal with elements $D$ and denote by
$\mathcal{A}={\sf diag}_{\{1,3\}}(D)$, if
\begin{align}
  \mathcal{A} [i_1, i_2, i_3] = \delta_{i_1 i_3} D[i_2, i_3],\quad \delta_{ij} = \begin{cases}
                                                                                   1 & i=j              \\
                                                                                   0 & \text{otherwise}
                                                                                 \end{cases}
\end{align}
Also for  two disjoint subsets
$ S = \{s_1, \cdots, s_t\} $ and $ Q = \{q_1, \cdots, q_{t'}\} $ of modes $ \{1, \cdots, N\} $. $ \mathcal{A} \in \mathbb{R}^{I_1 \times \cdots \times I_N} $ is called $ \{S,Q\} $-diagonal if $ A_{i_1 \cdots i_N} $ can be nonzero only if $ i_{s_1} = \cdots = i_{s_t} $ and $ i_{q_1} = \cdots = i_{q_{t'}} $.

In the following we will face with structured circulant tensors which defined in \cite{pboundary} as follows:
$\mathcal{A} \in \mathbb{R}^{I_{1} \times I_{2} \times \cdots \times I_{N}} $ is called $ \{ l,k \} $-circulant tensor, if $ I_{l} = I_{k} = n $, and
\[ \mathcal{A}[:,\cdots,:,i_{l},:, \cdots,:, i_{k}, :, \cdots, :] = \mathcal{A}[:,\cdots,:,i'_{l},:, \cdots,:, i'_{k}, :, \cdots, :] ,\quad  \text{if}\quad i_{l} - i_{k} \equiv i'_{l}- i'_{k} \quad(\text{mod} \quad  n ).
\]
The authors in \cite{pboundary} showed that such structured tensor can be diagonalized partially by FFT, which causes fast computation of some contraction products with such tensors.

\subsection{t-Product and t-SVD}\label{subsec2_2}
For the first time, tensor-tensor product (t-Product)  for $ 3 $-order tensors
$ \mathcal{A} \in \mathbb{R}^{n_1 \times \ell \times n_3} $ and $ \mathcal{B} \in \mathbb{R}^{\ell \times n_2 \times n_3} $ is introduced  in \cite{tSVD} as follows:
\begin{align*}
    \mathbb{R}^{n_1\times n_2\times n_3}\ni \mathcal{C} = \mathcal{A} \star_t \mathcal{B} = \mathsf{fold}\left(\mathsf{circ}(\mathsf{unfold}(\mathcal{A})).\mathsf{unfold}(\mathcal{B})\right) ,
\end{align*}
which
\begin{align*}
    \mathsf{circ}( \mathsf{unfold}(\mathcal{A})) = \begin{bmatrix}
                                                       {{\boldsymbol{A}_{1}}}     & {{\boldsymbol{A}_{{n_3}}}}     & \cdots & {{\boldsymbol{A}_{2}}} \\
                                                       {{\boldsymbol{A}_{2}}}     & {{\boldsymbol{A}_{1}}}         & \cdots & {{\boldsymbol{A}_{3}}} \\
                                                       \vdots                     & \vdots                         & \ddots & \vdots                 \\
                                                       {{\boldsymbol{A}_{{n_3}}}} & {{\boldsymbol{A}_{{n_3} - 1}}} & \cdots & {{\boldsymbol{A}_{1}}}
                                                   \end{bmatrix} \in \mathbb{R}^{n_1 n_3 \times \ell n_3},
\end{align*}
is a block circulant matrix.
Also,
$ \mathsf{fold} $ and $ \mathsf{unfold} $ are the following operations:
\begin{align*}
    \mathsf{unfold}(\mathcal{A}) = \begin{bmatrix}
                                       {{\boldsymbol{A}_{1}}} \\
                                       {{\boldsymbol{A}_{2}}} \\
                                       \vdots                 \\
                                       {{\boldsymbol{A}_{{n_3}}}}
                                   \end{bmatrix} \in \mathbb{R}^{n_1 n_3 \times \ell},~~\mathsf{fold}(\mathsf{unfold}(\mathcal{A})) = \mathcal{A}.
\end{align*}

Based on the properties of block circulant matrices, t-Product can be efficiently computed using the Fast Fourier Transform (FFT) algorithm, as follows:
\begin{align*}
  \mathsf{unfold}(\mathcal{C}) = ( \boldsymbol{F}_{n_{3}}^{*} \otimes I_{n_{1}} )\left ( ( \boldsymbol{F}_{n_{3}} \otimes I_{n_{1}} ) \mathsf{circ}\left( \mathsf{unfold}(\mathcal{A})\right)  ( \boldsymbol{F}_{n_{3}}^{*} \otimes I_{\ell} ) \right) ( \boldsymbol{F}_{n_{3}} \otimes I_{\ell} ) \mathsf{unfold}(\mathcal{B}).
\end{align*}
Based on the t-Product operator \cite{tSVD}, proved that on arbitrary 3-order tensors $\mathcal{A}\in \mathbb{R}^{n_1\times n_2\times n_3}$, can be decompoed as follows
\begin{align}\label{t-SVDKilmer}
    \mathcal{A} = \mathcal{U} \star_t \mathcal{S} \star_t \mathcal{V}^{\mathsf{T}},
\end{align}
where $ \mathcal{U} \in \mathbb{R}^{n_{1} \times n_{1} \times n_{3}} $ and $ \mathcal{V} \in \mathbb{R}^{n_{2} \times n_{2} \times n_{3}} $ are orthogonal, and $ \mathcal{S} \in \mathbb{R}^{n_{1} \times n_{2} \times n_{3}} $ is a f-diagonal tensor (according to our definition $\{1,2\}$-diagonal) \cite{tSVD}. Also, the transpose $\mathcal{V}^{\sf T}$ in \eqref{t-SVDKilmer} is defined in \cite{tSVD}.
This decomposition is an extension of SVD for $3$-order tensors named t-SVD.
The scheme of this decomposition is shown in Fig. \ref{fig1}.
\begin{figure}[ht]
    \centering
    \includegraphics[width=0.8\linewidth]{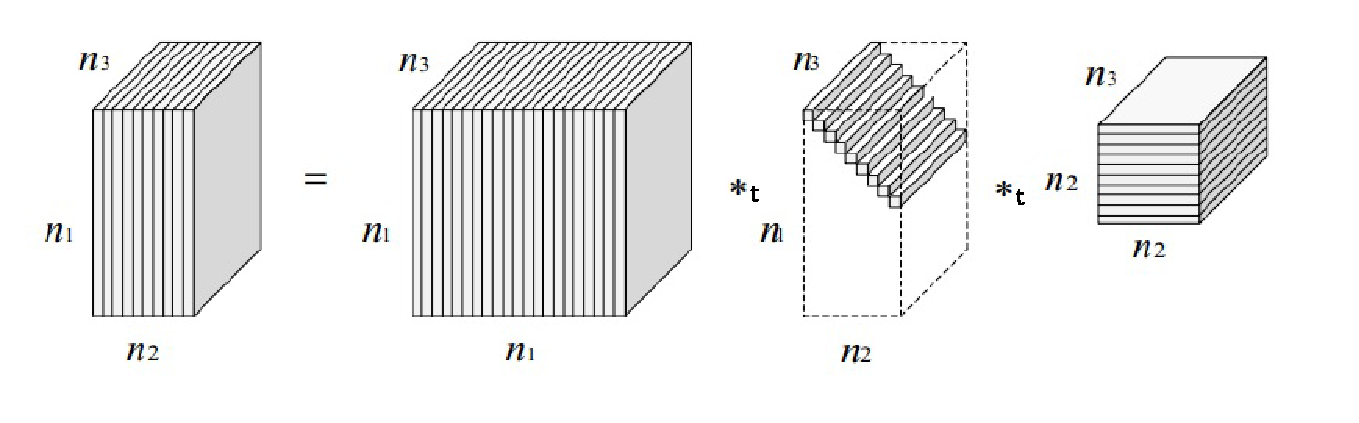}
    \caption{
        t-SVD scheme of an $ n_1 \times n_2 \times n_3 $ tensor \cite{Kilmer2},\cite{TRPCA}.
    }
    \label{fig1}
\end{figure}

Furthermore, when $n_{3} = 1$, the t-SVD reduces to the SVD of a matrix. The t-SVD, similar to SVD, has various applications including dimension reduction, classification, tensor completion, and even in deep neural networks \cite{wu2022robust}. In \cite{n-dim-kilmer}, the authors extended the t-Product and t-SVD to arbitrary order tensors in 2013.

\section{t-Product with signal processing viewpoint}\label{sec3}
This section presents a new perspective on the t-Product and demonstrates that it can be interpreted as a block convolution between two tensors. This perspective enables us to define a new tensor-tensor product with better properties than the t-Product. Another advantage of this approach is that it provides a simple framework for computing the t-SVD for higher dimensions, which is easier to implement.

Convolution is a fundamental concept in signal and image processing with many applications such as edge detection \cite{edge1, edge2}, image enhancement and deblurring \cite{blur2, blur1}. Convolution is also the basis of outstanding convolutional neural network architectures in deep learning \cite{GoodBengCour16}.

Mathematically, the convolution between two vectors $ \boldsymbol{x, h} \in \mathbb{R}^{n} $ is defined as following:
\begin{align}\label{convolutionVector}
  y_{i} = \sum h_{i-j}x_{j}, ~i = 1,\cdots,n.
\end{align}
Here, $\boldsymbol{x}$,  $\boldsymbol{y}$,
are considered as exact and filtered signals respectively and $\boldsymbol{h}$ denotes the filter kernel. In the convolution process, computing some elements of $\boldsymbol{y}$ requires signals from the boundary of $\boldsymbol{x}$. Various approaches have been proposed to deal with the boundary of $\boldsymbol{x}$, such as using zero (black) boundary {\cite{zboundary}}, repetition (periodic) boundary {\cite{pboundary}}, reflective boundary {\cite{rboundary}}, etc., which are known as artificial boundary conditions (BC). Here, we show that a block convolution with a periodic boundary condition provides insight into the t-Product from a signal processing viewpoint.

Let $ \boldsymbol{x} = [x_1, x_2, \cdots, x_n] \in \mathbb{R}^{n} $ and $ \boldsymbol{h} = [h_0, h_1, \cdots, h_{n-1}] = [a_1, a_2, \cdots, a_n] = \boldsymbol{a} \in \mathbb{R}^{n} $, became the input signal and nonzero part of the kernel $ \{h_i\} $, respectively. By considering the periodic boundary condition for signal $ \boldsymbol{x} $ in convolution process \eqref{convolutionVector}, the out side of $ \boldsymbol{x}=[x_1, x_2, \cdots, x_n] $ will be
\begin{align}\label{rez1}
  x_i=
  \begin{cases}
    x_{n+i}, & i=1-n,\cdots,0  \\
    x_{i-n}, & i=n+1,\cdots,2n
  \end{cases},
\end{align}
or schematically  as shown in Fig. \ref{figpb}.
\begin{figure}[!ht]
  \centering
  \includegraphics[width=0.8\linewidth]{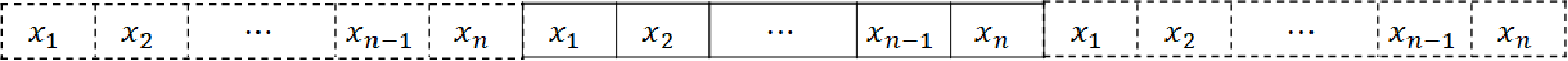}
  \caption{Scheme of the periodic boundary condition. The dashed part is the periodic boundary of signal $ \boldsymbol{x} $.}\label{figpb}
\end{figure}
By this assumption, $ y_i $ for $ i = 1, \cdots, n $ will be
\begin{align}\label{eq4}
  y_i = \sum_{j = 1}^{i} a_{i-j+1}x_{j} + \sum_{j = i+1}^{n} a_{n+1+i-j}x_{j}, \qquad i = 1, \cdots, n.
\end{align}
By defining the circulant  matrix $ \mathsf{circ}(\boldsymbol{a}) \in \mathbb{R}^{n \times n} $ as follows:
\begin{align}\label{Rez2}
  \mathsf{circ}(\boldsymbol{a})[i,j] = a_{k}, \quad \mathrm{where} \quad k =
  \begin{cases}
    i-j+1,   & i \geq j \\
    n+1+i-j, & i < j
  \end{cases}, \qquad i,j = 1, \cdots, n,
\end{align}
the equation \eqref{eq4} will be:
\begin{align}\label{bcirc}
  \begin{bmatrix}
    y_1    \\
    y_2    \\
    \vdots \\
    y_n
  \end{bmatrix} =
  \begin{bmatrix}
    {{a_{1}}}   & {{a_{{n}}}}     & \cdots & {{a_{2}}} \\
    {{a_{2}}}   & {{a_{1}}}       & \cdots & {{a_{3}}} \\
    \vdots      & \vdots          & \ddots & \vdots    \\
    {{a_{{n}}}} & {{a_{{n} - 1}}} & \cdots & {{a_{1}}}
  \end{bmatrix}
  \begin{bmatrix}
    x_1    \\
    x_2    \\
    \vdots \\
    x_n
  \end{bmatrix}
  \quad \mathrm{or} \quad\boldsymbol{y} = \mathsf{circ}(\boldsymbol{a})\boldsymbol{x}.
\end{align}

According to the properties of the circulant matrix \cite{pboundary}, we will have $ \boldsymbol{y} = \mathsf{ifft}(\mathsf{fft}(\boldsymbol{a})\odot \mathsf{fft}(\boldsymbol{x})) $, which shows that $ \boldsymbol{y} $ could be computed very fast by FFT without the need to construct $ \mathsf{circ}(\boldsymbol{a}) $.
Here $\odot$ denotes Hadamard (point-wise) product.

Now, if the elements $a_i, x_i$ and $y_i$ in convolution with periodic BC, are substituted with matrices
$\mathcal{A}[:,i,:]=A_{i} \in \mathbb{R}^{n_1\times n_3}$,
$\mathcal{X}[:,i,:]=X_i \in \mathbb{R}^{n_3\times n_4}$, and $\mathcal{Y}[:,i,:]= Y_i \in \mathbb{R}^{n_1\times n_4}, ~ i=1,\cdots, n_2$, the block version of equations \eqref{rez1} and \eqref{eq4} will be
\begin{align}\label{Rez3}
  X_{i} =
  \begin{cases}
    X_{n_{2}+i}, & i = 0, \cdots, 1-n_{2}      \\
    X_{i-n_{2}}, & i = n_{2}+1, \cdots, 2n_{2}
  \end{cases},
\end{align}
and
\begin{align}\label{Y_conv}
  \mathcal{Y}[:,i,:] = \sum_{j = 1}^{i} \mathcal{A}[:, i-j+1, :]X_{j} + \sum_{j = i+1}^{n_{2}} \mathcal{A}[:, n_{2}+1+i-j,:]X_{j}, \qquad i = 1, \cdots, n_{2}.
\end{align}
So, this equation is the bock convolution with periodic BC of tensor $\mathcal{A}\in \mathbb{R}^{n_1\times n_2\times n_3}$ and $\mathcal{X}\in \mathbb{R}^{n_3\times n_2\times n_4}$ in mode-2 which gives $\mathcal{Y}\in \mathbb{R}^{n_1\times n_2 \times n_4}$.
Similar to definition of $\mathsf{circ}(\boldsymbol{a})$ in \eqref{Rez2}, by defining a $ 4 $-order tensor $ \text{Tcirc}(\mathcal{A}) \in \mathbb{R}^{n_1 \times n_{2} \times n_3 \times n_{2}} $, with elements,
\begin{align}\label{eq_plus}
  \text{Tcirc}(\mathcal{A})[:,i,:,j] = \mathcal{A}[:,k,:], \quad \mathrm{where} \quad k =
  \begin{cases}
    i-j+1,       & i \geq j \\
    n_{2}+1+i-j, & i < j
  \end{cases}, \qquad i,j = 1, \cdots, n_{2},
\end{align}
the equation \eqref{Y_conv} becomes:
\begin{align}\label{Y_tCirc}
  \mathcal{Y}[:,i,:] = \sum_{j=1}^{n_{2}} \text{Tcirc}(\mathcal{A})[:,i,:,j]\mathcal{X}[:,j,:], \qquad i =1,\cdots,n_{2}.
\end{align}
Which is equal to the following contraction:
\begin{align}\label{(R)}
  \mathcal{Y} = \left \langle \text{Tcirc}(\mathcal{A}), \mathcal{X} \right \rangle_{3:4;1:2}.
\end{align}
between $ \text{Tcirc}(\mathcal{A}) $ and $ \mathcal{X} $,
Furthermore, by folding this equation, we get the following equation
\begin{align}\label{BCMatrix}
  \begin{bmatrix}
    Y_1    \\
    Y_2    \\
    \vdots \\
    Y_n
  \end{bmatrix} =
  \begin{bmatrix}
    {{A_{1}}}   & {{A_{{n}}}}     & \cdots & {{A_{2}}} \\
    {{A_{2}}}   & {{A_{1}}}       & \cdots & {{A_{3}}} \\
    \vdots      & \vdots          & \ddots & \vdots    \\
    {{A_{{n}}}} & {{A_{{n} - 1}}} & \cdots & {{A_{1}}}
  \end{bmatrix}
  \begin{bmatrix}
    X_1    \\
    X_2    \\
    \vdots \\
    X_n
  \end{bmatrix}.
\end{align}
This shows that our tensor viewpoint is equal to the defined block form in \cite{tSVD} of t-product only by reordering of modes 2 and 3.

The tensor $ \text{Tcirc}(\mathcal{A}) $ according to definition of circulant tensors in \cite{pboundary}, is an $\{2,4\}$- circulant tensor, which enables us to calculate the contraction \eqref{(R)}, without constructing $ \text{Tcirc}(\mathcal{A}) $ and
very fast by FFT on mode-2 of tensors $\mathcal{A}, \mathcal{X}$ as follows:
\begin{lemma}
  The equation $ \mathcal{Y} = \langle \text{Tcirc}(\mathcal{A}), \mathcal{X} \rangle_{3{:}4;1{:}2} $ could be computed very fast by fast Fourier transform as follows:
  \begin{align*}
    \mathcal{Y} = \left( F^{*} \right)_{3}.{\mathcal{\bar{Y}}}
  \end{align*}
  where $ {\mathcal{\bar{Y}}}[:,i,:] = {\mathcal{\bar{A}}}[:,i,:]{\mathcal{\bar{X}}}[:,i,:] $, $ {\mathcal{\bar{A}}} = \left( F \right)_{3}.\mathcal{A} $, and $ {\mathcal{\bar{X}}} = \left( F \right)_{3}.\mathcal{X} $, and $F$, $F^{*}$ are Fourier transform and its conjugate transpose, respectively.
\end{lemma}
\begin{proof}
  This can be proved  based on Theorem 5.1 and Corollary 5.4 in \cite{pboundary}.
\end{proof}
This notation and representation of t-product by tensor form \eqref{(R)} give us a viewpoint that comfort the extension of t-product to arbitrary order.

Now, if we substitute the block components $A_i, X_i$ of Block convolution \eqref{Y_conv} and its corresponding boundary \eqref{Rez3} by the slices $\mathcal{A}[:,i_2,\cdots,i_{N-1},:]$ and
$\mathcal{X}[:,i_2,\cdots,i_{N-1},:]$ of tensors $ \mathcal{A} \in \mathbb{R}^{I_{1} \times I_{2} \times \cdots \times I_{N-1} \times J} $, $ \mathcal{X} \in \mathbb{R}^{J \times I_{2} \times \cdots \times I_{N}} $, we have $\{2,\cdots,N-1\}-$period padding of $\mathcal{X}$ as follows
\begin{align}\label{eq_Kj}
  \mathcal{X}[:, i_{2}, i_{3}, \cdots, i_{N-1}, :] = \mathcal{X}[:, s_{2}, s_{3}, \cdots, s_{N-1},:],\quad
  s_{j} =
  \begin{cases}
    I_{j} + i_{j}, & i_{j} = 0, \cdots, 1- I_{j}     \\
    i_{j} - I_{j}, & i_{j} = I_{j}+1, \cdots, 2I_{j}
  \end{cases}.
\end{align}
and convolution \eqref{Y_conv}  will be
\begin{align}\label{Nbconv2}
  \mathcal{Y}[:,\bar{i},:] = \sum_{j_2,\cdots,j_{N-1}}^{I_2,\cdots,I_{N-1}} \text{Tcirc}(\mathcal{A})[:,\bar{i},:,\bar{j},:]\mathcal{X}[:,\bar{j},:],
\end{align}
where  $ \bar{i} = i_{2}, \cdots, i_{N-1}  $, $ \bar{j} =  j_{2}, \cdots, j_{N-1}  $. Here $ \text{Tcirc}(\mathcal{A}) \in \mathbb{R}^{I_{1} \times I_{2} \times \cdots \times I_{N-1} \times J \times I_{2} \times \cdots \times I_{N-1}} $ defined as
\begin{align}
  \text{Tcirc}(\mathcal{A})[:,\bar{i},:\bar{j}] = \mathcal{A}[:, \bar{k}, :],\qquad \bar{k}=k_2,\cdots,k_{N-1}, \quad k_l=
  \begin{cases}
    i_l-j_l+1,       & i_l \geq j_l \\
    n_{l}+1+i_l-j_l, & i_l < j_l
  \end{cases},  i_l,j_l = 1, \cdots, n_{l},
\end{align}
is a $ \{ 2{:}N-1;N+1{:}2N-2 \} $-circulant tensor based on definition of circulant tensor in \cite{pboundary}. So, we call \eqref{Nbconv2} as a $\{2,\cdots,N-1\}-$mode block convolution of tensors $ \mathcal{A} \in \mathbb{R}^{I_{1} \times I_{2} \times \cdots \times I_{N-1} \times J} $, $ \mathcal{X} \in \mathbb{R}^{J \times I_{2} \times \cdots \times I_{N}} $, with periodic padding in modes $\{2,\cdots,N-1\}$, which gives $\mathcal{Y}\in \mathbb{R}^{I_1\times I_2\times \cdots \times I_{N-1}\times I_N}$.
It's clear that
the block convolution \eqref{Nbconv2} is equal to the following contraction:
\begin{align}\label{Ntproduct}
 \mathcal{A}\star_t \mathcal{X}:= \mathcal{Y}= \langle \text{Tcirc}(\mathcal{A}), \mathcal{X} \rangle_{N{:}2(N-1);1{:}N-1}
\end{align}
Therefor, we define the t-product of two N-order tensors
$ \mathcal{A} \in \mathbb{R}^{I_{1} \times I_{2} \times \cdots \times I_{N-1} \times J} $, $ \mathcal{X} \in \mathbb{R}^{J \times I_{2} \times \cdots \times I_{N}} $ as equation \eqref{Ntproduct}. By some mathematical manipulation, we could show that this arbitrary order t-product definition is equal to extension of t-product in \cite{n-dim-kilmer}. But our viewpoint gives an instrument that computing this t-product becomes more straightforward in notation, complexity, and computation than the definition in \cite{n-dim-kilmer}.

Since $ \text{Tcirc}(\mathcal{A}) $ is a $ \{ 2{:}N-1;N+1{:}2N-2 \} $-circulant tensor from \cite{pboundary}, we could prove that equation \eqref{Ntproduct}
could be done without construction of $ \text{t-circ}(\mathcal{A}) $ only by applying
of FFT on tensors $\mathcal{A}$ and $\mathcal{X}$ as shown in Algorithm \ref{n-t-Product-Us}.
\begin{algorithm}[ht]
  \caption{t-Product for arbitrary order tensors}\label{n-t-Product-Us}
  \begin{algorithmic}[1]
    \Require  $\mathcal{A}\in \mathbb{R}^{I_{1} \times I_{2} \times \cdots \times I_{N-1}\times J} $,
    $ \mathcal{X}\in \mathbb{R}^{J\times I_{2} \times  I_{3} \times \cdots \times I_{N}} $
    \Ensure $\mathcal{Y}\in \mathbb{R}^{ I_{1} \times \cdots \times I_{N}}$
    \State $ {\bar{\mathcal{A}}}=\left( F, \cdots, F \right)_{2:N-1}.\mathcal{A}$,
    $ {\bar{\mathcal{Y}}}=\left( F, \cdots, F \right)_{2:N-1}.\mathcal{Y} $
    \State
    ${\mathcal{\bar{Y}}}[:,\bar{i},:] = {\mathcal{\bar{A}}}[:,\bar{i},:]{\mathcal{\bar{X}}}[:,\bar{i},:],\quad \bar{i}=i_2,\cdots,i_{N-1},i_k=1,\cdots, I_k $
    \State ${\mathcal{Y}} = \left( F^{*}, \cdots, F^{*} \right)_{2:N-1}.{\mathcal{\bar{Y}}}$
  \end{algorithmic}
\end{algorithm}
\begin{theorem}[t-SVD for arbitrary order tensors]
  For each tensor $\mathcal{A}\in \mathbb{R}^{I_1\times \cdots \times I_N}$ with arbitrary order $N$, There exists the following decomposition of:
  \[ \mathcal{A}=\mathcal{U}\star_t \mathcal{S}\star_t \mathcal{V}^{T}\]
\end{theorem}
\begin{proof}
  The approach for proof is similar to theorem \ref{TcSVDTh} that we will prove in the following.
\end{proof}

\section{Novel tensor product and decomposition based on block convolution with reflective boundary conditions}\label{sec4}
The previous section has shown that the t-Product is a type of block convolution that utilizes periodic boundary conditions between two tensors. The concept of convolution has a rich history in the field of signal and image processing, and has been applied in various contexts such as image enhancement, image restoration, and deep convolutional neural networks (CNNs). For example, in image enhancement, a mask is convolved with an image for denoising or edge detection purposes, while in image restoration, the blurring process can be modeled as the convolution of the blurring mask and the input image \cite{zboundary}.

Although convolution masks must operate on the image boundary, it's not always possible to access the actual boundaries in all of these applications. As a result, various artificial boundary conditions (BCs) have been proposed and utilized in multiple scenarios. Examples of well-known BCs in the literature of image restoration include zero, periodic, reflective, and anti-reflective BCs. Experimental results have shown that reflective BCs generally outperform periodic and zero BCs due to their ability to model more complex scenarios that are more compatible with real situations \cite{rboundary}.

This section introduces a new type of tensor-tensor product based on block convolution between two tensors with reflective boundary conditions. The linear operator that corresponds to convolution by reflective BC can be diagonalized with a Cosine operator, which always produces real numbers unlike the FFT. Moreover, the complexity of this transformation is less than that of the FFT.

To extend this concept, we first describe convolution with reflective BCs and then introduce block matrix and block tensor versions of this operation, following a similar approach as in the previous section.

In convolution equation \eqref{convolutionVector} Let $ \boldsymbol{x}=[x_1,\dots,x_n]^{\mathsf{T}} $  and $ \boldsymbol{h}=[h_{1-n},\dots,h_1,\dots,h_{n-1}]^{\mathsf{T}}=[a_n ,\dots,a_1,\dots,a_n]^{\mathsf{T}} \in R^{2n-1} $ be the input signal and the nonzero part of a symmetric kernel $ \boldsymbol{h} $. Also consider reflective BC in outside the real signal $ \boldsymbol{x} $ , i.e.
\begin{align*}
    x_i=
    \begin{cases}
        x_{1-i}    & i=0,\cdots,1-n  \\
        x_{2n-i+1} & i=n+1,\cdots,2n
    \end{cases}
\end{align*}
By these assumptions the convolution \eqref{convolutionVector}, becomes
\begin{align}\label{eqThx}
    y_{i} = \sum_{j=1}^{n} \left( a_{|i-j|+1}+(1-\delta_{i+j,n+ 1})
    \begin{cases}
        a_{i+j}          & i+j\leq n \\
        a_{2(n+1)-(i+j)} & i+j> n
    \end{cases}
    \right) x_{j}, \quad i=1,\cdots,n.
\end{align}
where
\begin{align*}
    \delta_{i,j}=
    \begin{cases}
        1   & i=j  \\
        0 & i\neq j
    \end{cases}
\end{align*}
By defining $ \mathsf{Th}(\boldsymbol{a}) \in \mathbb{R}^{n \times n} $ matrix as following:
\begin{align}\label{eqTh}
    \mathsf{Th}(\boldsymbol{a})[i,j] = a_{|i-j|+1}+(1-\delta_{i+j,n+ 1})
    \begin{cases}
        a_{i+j}          & i+j\leq n \\
        a_{2(n+1)-(i+j)} & i+j> n
    \end{cases},\quad i,j=1,\cdots,n
\end{align}
the equation \eqref{eqThx} becomes to the following linear equation:
\begin{align}\label{rezth1}
    \boldsymbol{y} = \mathsf{Th}(\boldsymbol{a})\boldsymbol{x},
\end{align}
$\mathsf{Th}(\boldsymbol{a})$ has the following form in detail:
\begin{align}\label{eq_Th}
    \mathsf{Th}(\boldsymbol{a}) & =
    \begin{bmatrix}
        a_{1}   & a_{2}   & \cdots & a_{n-1} & a_{n}   \\
        a_{2}   & a_{1}   & \cdots & a_{n-2} & a_{n-1} \\
        \vdots  & \ddots  & \ddots & \ddots  & \vdots  \\
        a_{n-1} & a_{n-2} & \cdots & a_{1}   & a_{2}   \\
        a_{n}   & a_{n-1} & \cdots & a_{2}   & a_{1}   \\
    \end{bmatrix}+
    \begin{bmatrix}
        a_{2}  & a_{3}  & \cdots & a_{n}  & 0      \\
        a_{3}  & \adots & \adots & \adots & a_{n}  \\
        \vdots & \adots & 0      & \adots & \vdots \\
        a_{n}  & \adots & \adots & \adots & a_{3}  \\
        0      & a_{n}  & \cdots & a_{3}  & a_{2}  \\
    \end{bmatrix}                                                 \\
                                & =\mathsf{Toep}(\boldsymbol{a})+\mathsf{Hank}(\boldsymbol{a})
\end{align}
which shows that the $ \mathsf{Th}(\boldsymbol{a}) $ is a spacial kind of Toeplitz-plus-Hankel matrix,\cite{rboundary}.

\begin{figure}[ht]
    \centering
    \begin{subfigure}[t]{0.3\textwidth}
        \centering
        \includegraphics[width=0.95\textwidth]{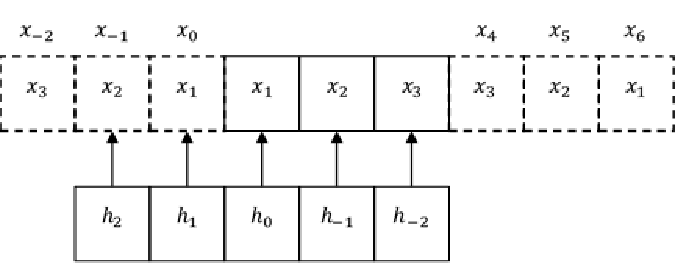}
        \caption{$ y_1 $}
    \end{subfigure}
    \begin{subfigure}[t]{0.3\textwidth}
        \centering
        \includegraphics[width=0.95\textwidth]{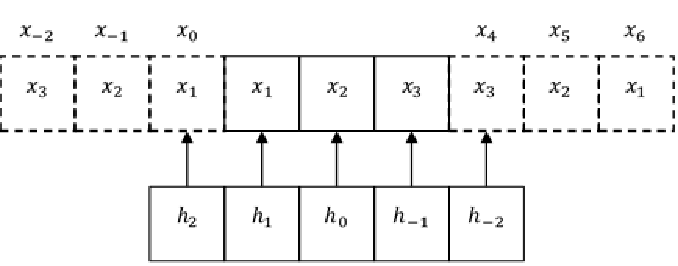}
        \caption{$ y_2 $}
    \end{subfigure}
    \begin{subfigure}[t]{0.3\textwidth}
        \centering
        \includegraphics[width=0.95\textwidth]{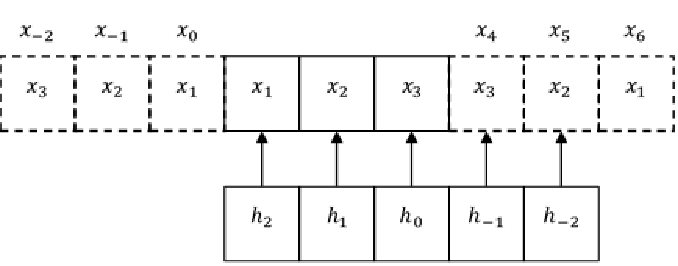}
        \caption{$ y_3 $}
    \end{subfigure}

    \caption{
        Convolution with with reflective boundary}\label{fig3}
\end{figure}

For example,
figure \ref{fig3}  demonstrates this BC schematically for $ n=3 $.
In \cite{rboundary} it has been shown that for any $ \boldsymbol{a} \in \mathbb{R}^{n} $, Toeplitz-plus-Hankel matrix in the $ \mathsf{Th}(\boldsymbol{a}) $ form, can be diagonalized by Cosine transform as{:}
\begin{align}\label{cosineTransformDiagonalize}
    \mathsf{Th}(\boldsymbol{a}) = {\overline{C}}^{\mathsf{T}}\boldsymbol{\Lambda}\overline{C},
    \quad
    \boldsymbol{\Lambda}=\mathsf{diag}(\bar{\boldsymbol{a}}),
    \quad
    \bar{\boldsymbol{a}}= \mathsf{diag}(c_1)~\overline{C}~\mathsf{Th}_1(\boldsymbol{a}),
    \qquad c_1 = 1 ./ {\overline{C}}[{:},1]
\end{align}
where $ {\overline{C}} \in \mathbb{R}^{n \times n} $ is discrete cosine matrix defined as
\begin{align*}
    {\overline{C}}\left[ i,j \right] = \sqrt{\dfrac{2 - \delta_{i1}}{n}}\cos\left( \dfrac{(i-1)(2j-1)}{2n}\pi \right), ~i,j = 1,\cdots,n
    ,\qquad
    \delta_{ij} = \begin{cases}
                      1,\quad i=j \\
                      0,\quad \mathrm{otherwise}
                  \end{cases},
\end{align*}
and
\begin{align*}
    \mathsf{Th}_1(\boldsymbol{a})= \begin{bmatrix}
                                       {{a_1} + {a_2}}       \\
                                       {{a_2} + {a_3}}       \\
                                       \vdots                \\
                                       {{a_{n - 1}} + {a_n}} \\
                                       a_n
                                   \end{bmatrix}= \boldsymbol{\Gamma} \boldsymbol{a}
    ,\qquad
    \boldsymbol{\Gamma} =\begin{bmatrix}
                             1      & 1      & 0      & \cdots & 0      \\
                             0      & 1      & 1      & \ddots & \vdots \\
                             0      & \ddots & \ddots & \ddots & 0      \\
                             \vdots & \ddots & 0      & 1      & 1      \\
                             0      & \cdots & 0      & 0      & 1
                         \end{bmatrix},
    ~ \boldsymbol{a}= \begin{bmatrix}
                          {{a_1}}       \\
                          {{a_2}}       \\
                          \vdots        \\
                          {{a_{n - 1}}} \\
                          a_n
                      \end{bmatrix},
\end{align*}
is the first column of $ \mathsf{Th}(\boldsymbol{a}) $.
Since
$
    \mathsf{diag}(c_1)~\boldsymbol{\Lambda}~{\mathsf{diag}(c_1)}^{-1} = \boldsymbol{\Lambda}
$,
diagonalization form \eqref{cosineTransformDiagonalize} could be rewritten as
\begin{align}\label{cosine4}
    \mathsf{Th}(\boldsymbol{a}) = {{C}}^{\raisebox{1pt}{$-1$}} \boldsymbol{\Lambda} {C}, \quad {C} =\mathsf{diag}(c_1)~\overline{C}, \quad \boldsymbol{\Lambda}=\mathsf{diag}(\bar{\boldsymbol{a}}), \quad \bar{\boldsymbol{a}}= C~\Gamma a.
\end{align}
where $C^{-1}=\overline{C}^{\sf T}\mathsf{diag}(1./c_1)$.
By this decomposition its clear that \eqref{rezth1} could be computed by fast cosine transform as follows
\begin{align}\label{rezth2}
    \boldsymbol{y} & =\mathsf{Th}(\boldsymbol{a})\boldsymbol{x}  ={C}^{\raisebox{1pt}{$-1$}}(\bar{\boldsymbol{a}}\odot\bar{\boldsymbol{x}}), \quad \bar{\boldsymbol{a}}={C}\Gamma\boldsymbol{a}, \quad \bar{\boldsymbol{x}}={C}\boldsymbol{x}
\end{align}
where $\odot$ denotes Hadamard(pointwise) product of vectors.
The equation \eqref{rezth2} is our base for defining new tensor product of tensors. But its clear that if we define
$\boldsymbol{y}=\boldsymbol{a}\star_c \boldsymbol{x}=\mathsf{Th}(\boldsymbol{a})\boldsymbol{x}$
obviously, this multiplication does not have the associative property, i.e.,
\[
    \boldsymbol{b}\star_c(\boldsymbol{a}\star_c \boldsymbol{\boldsymbol{x}})\neq (\boldsymbol{b}\star_c\boldsymbol{a})\star_c\boldsymbol{\boldsymbol{x}},
\]
which is not suitable for defining a new product operator. So we define
\[
    \boldsymbol{y}{:=}\boldsymbol{a}\star_c \boldsymbol{x}=\mathsf{Th}(\Gamma^{-1}\boldsymbol{a})\boldsymbol{x}
\]
clearly this product which is convolution (with reflective BC) between  $\widehat{\boldsymbol{a}}=\Gamma^{-1}\boldsymbol{a}$ and $x$ has associative property.

Similar to convolution with periodic BC in section \ref{sec3}, if we substitute the
elements $a_i, x_i$ and $y_i$ of convolution with reflective BC \eqref{eqThx}, with the  matrices $\widehat{\mathcal{A}}[:,i,:]=\widehat{A}_{i} \in \mathbb{R}^{n_1\times n_3}$, $\mathcal{X}[:,i,:]=X_i \in \mathbb{R}^{n_3\times n_4}$, and $\mathcal{Y}[:,i,:]= Y_i \in \mathbb{R}^{n_1\times n_4}$  for $i=1,\cdots, n_2$,
where
$\widehat{\mathcal{A}}=(\Gamma^{-1})_{2}.\mathcal{A},$ the mode-2 block reflective padding of $\mathcal{X}$ will be
\begin{align}\label{BBCR}
    X_i=\begin{cases}
            X_{1-i}      & i=0,\cdots,1-n_2    \\
            X_{2n_2-i+1} & i=n_2+1,\cdots,2n_2
        \end{cases}
\end{align}
and \eqref{eqTh} for $i=1,\cdots,n_2$ becomes
\begin{align}\label{M}
    \mathcal{Y}[:,i,:] = \sum_{j=1}^{n_2} \left( \widehat{\boldsymbol{A}}_{|i-j|+1}+(1-\delta_{i+j,n_2+ 1})
    \begin{cases}
        \widehat{\boldsymbol{A}}_{i+j}            & i+j\leq n_2 \\
        \widehat{\boldsymbol{A}}_{2(n_2+1)-(i+j)} & i+j> n_2
    \end{cases}
    \right) X[:,i,:].
\end{align}
So, this is the bock convolution with reflective-padding(BC) in mode-2  of tensors $\widehat{\mathcal{A}}\in \mathbb{R}^{n_1\times n_2\times n_3}$ and $\mathcal{X}\in \mathbb{R}^{n_3\times n_2\times n_4}$  which gives $\mathcal{Y}\in \mathbb{R}^{n_1\times n_2 \times n_4}$.
By defining the 4-order tensor $ \mathsf{TH}(\mathcal{A}) $ with elements
\begin{align}\label{eq_THA_e}
    \mathsf{TH}(\widehat{\mathcal{A}})[:,i_{2},:,i_{4}] = \widehat{\mathcal{A}}[:,|i_{2} - i_{4}|+1,:] + \left(1-\delta_{i_{2}+i_{4},n_2+1}\right)
   \begin{cases}
        \widehat{\mathcal{A}}[:, i_{2} + i_{4},:]              & i_{2}+i_{4} \leq n_2 \\
        \widehat{\mathcal{A}}[:, 2(n_2+1) - (i_{2} + i_{4}),:] & i_{2} + i_{4} > n_2
    \end{cases}.
\end{align}
equation \eqref{M} becomes
\begin{align*}
    \mathcal{Y}[:,i_{2},:] = \sum_{i_{4} = 1}^{n} \mathsf{TH}(\widehat{\mathcal{A}})[:,i_{2},: i_{4}]\mathcal{X}[:,i_{4},:],\quad i_2=1,\cdots,n_2
\end{align*}
which is equal to the following contraction:
\begin{align}\label{eq_product}
    \mathcal{Y} = \langle \mathsf{TH}(\widehat{\mathcal{A}}), \mathcal{X} \rangle_{3:4;1:2}.
\end{align}
\begin{defi}[$\star_c{}\text{-Product}$]
    We define $\star_c$ operation between
    $\mathcal{A}\in\mathbb{R}^{I_1\times I_2\times J}$ and $\mathcal{X}\in\mathbb{R}^{J\times I_2\times I_3}$ tensors, as
    \begin{align}\label{tcProduct}
        \mathcal{Y}=	\mathcal{A}\star_c\mathcal{X}=\langle \mathsf{TH}(\widehat{\mathcal{A}}),{\mathcal{X}}\rangle_{3{:}4,1{:}2}, \quad \widehat{A}=(\Gamma)_2.\mathcal{A}.
    \end{align}
\end{defi}
In the following we show that this $\star_c$ product can be done very fast by Fast cosine trasform on mode-2 of tensors $\mathcal{A}$ and $\mathcal{X}$, without constraucting tensors $\widehat{A}$ and $\mathsf{TH}(\widehat{\mathcal{A}})$.

From \eqref{eq_THA_e} it's clear that $ \mathsf{TH}(\widehat{\mathcal{A}})[i_{1},:,i_{3},:] $ is a Toeplitz plus Hankel matrix, in the form \eqref{eq_Th}, constructed by
$\widehat{\mathcal{A}}[i_1,:,i_3]$ vector. so
\begin{align}\label{eq_THA}
    \mathsf{TH}(\widehat{\mathcal{A}})[i_{1},:,i_{3},:] = \mathsf{Th}(\widehat{\mathcal{A}}[i_{1},:,i_{3}])
\end{align}
where $\mathsf{Th}(\mathcal{A}[i_{1},:,i_{3}])$ has the structure like \eqref{cosineTransformDiagonalize}.
Similar to definition of circulant and Toeplitz tensor in \cite{pboundary}, we say that $ \mathsf{TH}(\widehat{\mathcal{A}}) $ is an $ \{2,4\} $-Toeplitz-plus-Hankel tensor(and for simplicity in abbreviation say $ \mathsf{TH}(\widehat{\mathcal{A}}) $  is $\{2,4\}$-{\sf TH} tensor).
\begin{lemma}\label{lemm_diagonal}
    The $\{2,4\}-${\sf TH} tensor $\mathsf{TH}(\widehat{\mathcal{A}})$, where $\widehat{\mathcal{A}}=(\Gamma^{-1})_{2}.\mathcal{A}$ could be decomposed as follows:
    \begin{align}\label{eq_IR}
        \mathsf{TH}(\widehat{\mathcal{A}})=({C^{-1}},{C^{\sf T}})_{2,4}.\mathcal{D},
    \end{align}
    where $\mathcal{D}={\sf diag}_{\{2,4\}}(\overline{\mathcal{A}})$ is a
    $\{\{2,4\}\}$-diagonal tensor with elements,
    \begin{align}\label{eq_IIR}
        \mathcal{D}[i_1,i_2,i_3,i_4]=\delta_{i_2,i_2}{\overline{\mathcal{A}}}[i_1,i_2,i_3],\quad  \overline{\mathcal{A}}=(\overline{C})_2.\mathcal{A}
    \end{align}
\end{lemma}
\begin{proof}
    From \eqref{eq_THA} $\mathsf{TH}(\widehat{\mathcal{A}})[i_{1},:,i_{3},:] = \mathsf{Th}(\widehat{\mathcal{A}}[i_{1},:,i_{3}])$
    is a Toeplitz-plus-Hankel matrix and so by decomposition \eqref{cosine4} and $\widehat{\mathcal{A}}=(\Gamma^{-1})_{2}.{\mathcal{A}}$, we have
    \begin{align}\label{eq_2}
        \mathsf{TH}(\widehat{\mathcal{A}})[i_{1},:,i_{3},:] ={C}^{\raisebox{1pt}{$-1$}}\Lambda^{(i_1,i_3)}{C}=\left(C^{-1},C^{\sf T}\right)_{1,2}.\Lambda^{(i_1,i_3)}
    \end{align}
    where
    \begin{align}\label{eq:manl}
        \Lambda^{(i_1,i_3)}={\sf diag}({C}~\Gamma~ \widehat{\mathcal{A}}[i_1,:,i_3])= {\sf diag}({C}~\mathcal{A}[i_1,:,i_3])= {\sf diag}(\overline{\mathcal{A}}[i_1,:,i_3]).
    \end{align}
    Here $\Lambda^{(i_1,i_3)}(i_2,j_4)=\delta_{i_2i_4}\overline{\mathcal{A}}[i_1,i_2,i_3]$.
    Therefore, the tensor $\mathcal{D}$ defined as $   \mathcal{D}[i_1,{:},i_3,{:}]=\Lambda^{(i_1,i_3)}$
    satisfies
    \[
        \mathcal{D}[i_1,i_2,i_3,i_4]=\Lambda^{(i_1,i_3)}(i_2,i_4)=\delta_{i_2i_4}\overline{\mathcal{A}}[i_1,i_2,i_3],
    \]
    which proves \eqref{eq_IIR}.
    Also, from  definition of $\mathcal{D}$, the equation \eqref{eq_2} will be
    \begin{align}
        \mathsf{TH}(\widehat{\mathcal{A}})[i_{1},:,i_{3},:] =({C}^{\raisebox{1pt}{$-1$}},{C}^{\raisebox{1pt}{${\sf T}$}})_{2,4}.\mathcal{D}(i_1,:,i_3,:)
    \end{align}
    which proves \eqref{eq_IR}.
\end{proof}

Now,  the lemma \eqref{lemm_diagonal}, helps to develop a fast methods to compute $\star_c$ product
by cosine transform on mode-2 of tensors $\mathcal{A}$ and $\mathcal{X}$, without constructing tensors $\widehat{A}$ and $\mathsf{TH}(\widehat{\mathcal{A}})$ as following lemma.
\begin{lemma}
    The $\star_c$ product $
        \mathcal{Y}=	\mathcal{A}\star_c\mathcal{X}=\langle \mathsf{TH}(\widehat{\mathcal{A}}),{\mathcal{X}}\rangle_{3{:}4,1{:}2},
    $ in equation \eqref{tcProduct} could be computed by cosine transform as follows:
    \begin{align*}
        \mathcal{Y}=({C}^{\raisebox{1pt}{$-1$}})_2.\overline{\mathcal{Y}}
    \end{align*}
    where $\overline{\mathcal{Y}}[{:},i_2,{:}]=\overline{\mathcal{A}}[{:},i_2,{:}]~\overline{\mathcal{X}}[{:},i_2,{:}], \quad i_2=1,\cdots,n_2$ and
    \begin{align*}
        \overline{\mathcal{X}}=({C})_2.\mathcal{X},\quad \overline{\mathcal{A}}=({C})_2.\mathcal{A}
    \end{align*}
\end{lemma}
\begin{proof}
    Proof: From Lemma \ref{lemm_diagonal},
    $\mathsf{TH}(\widehat{\mathcal{A}})=({C}^{\raisebox{1pt}{$-1$}},{C}^{\raisebox{1pt}{${\sf T}$}})_{2,4}.\mathcal{D}$,
    where $\mathcal{D}=\text{diag}_{\{2,4\}}(\overline{\mathcal{A}})$, and $\overline{\mathcal{A}}=(C)_2.\mathcal{A}$.
    So
    \begin{align}\label{RezL}
        \nonumber
        \mathcal{Y} & =\langle \mathsf{TH}(\widehat{\mathcal{A}}),\mathcal{X}\rangle_{3{:}4,1{:}2}                                            \\
        \nonumber
                    & =\langle({C}^{\raisebox{1pt}{$-1$}},{C}^{\raisebox{1pt}{${\sf T}$}})_{2,4}.\mathcal{D},\mathcal{X}\rangle_{3{:}4,1{:}2} \\
                    & =({C}^{\raisebox{1pt}{$-1$}})_2.\left[\langle\mathcal{D},({C})_{2}.\mathcal{X}\rangle_{3{:}4,1{:}2}\right]
    \end{align}
    which the last equation comes immediately follow from the definitions of contraction and matrix-tensor product.
    If $\overline{\mathcal{X}}=({C})_2.\mathcal{X}$ and $\overline{\mathcal{Y}}=({C})_2.\mathcal{Y}$
    equation \eqref{RezL} becomes
    \begin{align}\label{RezYb}
        \overline{\mathcal{Y}}=\langle\mathcal{D},\overline{\mathcal{X}}\rangle_{3{:}4,1{:}2}.
    \end{align}
    Since $\mathcal{D}=\text{diag}_{\{2,4\}}(\overline{\mathcal{A}})$ is $\{2,4\}-$diagonal, \eqref{RezYb} gives,
    \begin{align*}
        \overline{\mathcal{Y}}[{:},i_2,{:}]=\overline{\mathcal{A}}[{:},i_2,{:}]~\overline{\mathcal{X}}[{:},i_2,{:}],
    \end{align*}
    which proves the Lemma.
\end{proof}
Algorithm \ref{Alg1} displays how to compute $\star_c{}\text{-Product}$ between two 3D tensors.

\begin{algorithm}[ht]
    \caption{$\star_c{}\text{-Product}$ for 3-order tensors}\label{Alg1}
    \begin{algorithmic}[1]
        \Require
        $ \mathcal{A}\in \mathbb{R}^{n_1\times n_2\times \ell} $, $ \mathcal{X}\in \mathbb{R}^{\ell\times n_2\times n_3} $
        \Ensure
        $\mathcal{Y}$
        \State Compute $ \overline{\mathcal{A}} = (\boldsymbol{C})_2 . \mathcal{A} $
        and $ \overline{\mathcal{X}} = (\boldsymbol{C})_2 . \mathcal{X} $
        \State Compute each mode-2 slice of $ \overline{\mathcal{Y}} $ by $ \overline{\boldsymbol{Y}}_i = \overline{\boldsymbol{A}}_i \overline{X}_i $, $i=1,\cdots,n_2$
        \State Compute $ \mathcal{Y} = (\boldsymbol{C}^{-1})_2 . \overline{\mathcal{Y}} $
    \end{algorithmic}
\end{algorithm}

\subsection{$\star_c$-Product Generalization for arbitrary-order tensors}
In this section we would like to extend the $\star_c$ product for for arbitrary order tensors  $ \mathcal{A} \in \mathbb{R}^{I_1 \times \dots \times I_{N-1} \times J} $ and $ \mathcal{X} \in \mathbb{R}^{J \times I_2 \times \dots \times I_N} $.
In the last section for 3-order, we defined $\star_c$ product of two tensors $\mathcal{A}\in$ and $\mathcal{X}\in$  as the Block convolution  with reflective  padding of mode-2 of $\widehat{A}$ and $X$, which leads to contraction \eqref{eq_product}.
Now, if we substitute the block components $\widehat{A}_i, X_i$ of Block convolution \eqref{M} with boundary \eqref{BBCR} by the slices $\widehat{\mathcal{A}}[:,i_2,\cdots,i_{N-1},:]$ and
$\mathcal{X}[:,i_2,\cdots,i_{N-1},:]$ of tensors
\[\widehat{\mathcal{A}}=
    (\Gamma,\cdots,\Gamma)_{2,\cdots,N-1}.{\mathcal{A}}
    \in \mathbb{R}^{I_{1} \times I_{2} \times \cdots \times I_{N-1} \times J}, \quad \mathcal{X} \in \mathbb{R}^{J \times I_{2} \times \cdots \times I_{N}}, \]
we have $\{2,\cdots,N-1\}-$reflective padding as follows:
\begin{align}
    \mathcal{X}[:, i_{2}, i_{3}, \cdots, i_{N-1}, :] = \mathcal{X}[:, s_{2}, s_{3}, \cdots, s_{N-1},:],\quad
    s_{j} =
    \begin{cases}
        1-i_j,        & i_{j} = 0, \cdots, 1- I_{j}     \\
        2I_{j}-i_j+1, & i_{j} = I_{j}+1, \cdots, 2I_{j}
    \end{cases}.
\end{align}
and the convolution \eqref{M}  will be
\begin{align}\label{BBM}
    \mathcal{Y}[:,\bar{i},:] = \sum_{j_2,\cdots,j_{N-1}}^{I_2,\cdots,I_{N-1}} \left (
    \widehat{\mathcal{A}}[:,{|\bar{i}-\bar{j}|}+1,:]+
    \prod_{k=2}^{N-1}(1-\delta_{i_k+j_k,n_k+1})\widehat{\mathcal{A}}[\bar{:,\ell,:}]
    \right)\mathcal{X}[:,\bar{j},:],
\end{align}
where  $ \bar{i} =  i_{2}, \cdots, i_{N-1}  $, $ \bar{j} = j_{2}, \cdots, j_{N-1} $ and
\begin{align}
    {|\bar{i}-\bar{j}|+1} & =|i_2-j_2|+1,|i_3-j_3|+1,\ldots,|i_{N-1}-j_{N-1}|+1 \\
    \bar{\ell}            & =\bar{\ell}_2,\ldots,\bar{\ell}_{N-1},\qquad
    \bar{\ell}_k =\begin{cases}
                      i_k+j_k        & i_k+j_k\leq I_k \\
                      2n_k-(i_k+j_k) & i_k+j_k>I_k
                  \end{cases}
\end{align}
Now by defining $ \mathsf{TH}(\widehat{\mathcal{A}}) $ as
\begin{align}\label{extendTHA}
    \mathsf{TH}(\widehat{\mathcal{A}})[{:},i_2,i_3,\ldots,i_{N-1},{:},j_2,j_3,\ldots,j_{N-1}]=\widehat{\mathcal{A}}[:,{|\bar{i}-\bar{j}|}+1,:]+\prod_{k=2}^{N-1}(1-\delta_{i_k+j_k,n_k+1})\widehat{\mathcal{A}}[:,\bar{\ell},:]
\end{align}
equation \eqref{M} leads to the following contraction product
\begin{align}
    \mathcal{Y}= \langle \mathsf{TH}(\widehat{\mathcal{A}}), \mathcal{X} \rangle_{N:2(N-1);1:N-1}.
\end{align}
\begin{defi}
    We define $\star_c$ operation between
    $\mathcal{A}\in\mathbb{R}^{I_1\times I_2\cdots \times I_{N-1}\times J}$ and $\mathcal{X}\in\mathbb{R}^{J\times I_2 \cdots \times \times I_N}$ tensors, as
    \begin{align}
        \mathcal{Y}= \mathcal{A}\star_c\mathcal{X}=\langle  \mathsf{TH}(\widehat{\mathcal{A}}), \mathcal{X} \rangle_{N:2(N-1);1:N-1}.
    \end{align}
\end{defi}
Here  $ \mathsf{TH}(\widehat{\mathcal{A}})\in \mathbb{R}^{I_1\times I_2\times \cdots\times I_N\times I_2\times I_N-1} $ is a $\{2{:}N-1;N+1{:}2(N-1)\}$-TH tensor and in the following we show that this tensor is diagonaizable by cosine transform in $\{2,\dots,N{-}1\}$-modes, which cause to find $x_c$ product very fast without construction of  $ \mathsf{TH}(\widehat{\mathcal{A}}) $ only by applying fast cosine transform in $\{2,\dots,N{-}1\}$-modes of tensors $\mathcal{A}$ and $\mathcal{X}$. To show this property we need an extension of (Hadamard) point-wise product for tensors in some modes as follows:

\begin{defi}
    We define the the $\{2,\cdots,N-1\}-$Hadamard (point-wise) product of order-N tensors  $\mathcal{S}\in \mathbb{R}^{I_1\times I_2\times I_{N-1}\times I_N}$ and $\mathcal{P}\in \mathbb{R}^{I_N\times I_2\times I_{N-1}\times J}$
    according to ${2,\cdots,N-1}$-modes as follows
    \begin{align*}
        \mathcal{T}=\mathcal{S}\odot_{2{:}N-1}\mathcal{P}
        \Rightarrow
        \mathcal{T}[{:},i_2,\ldots,i_{N-1},{:}]=\mathcal{S}[{:},i_2,\ldots,i_{N-1},{:}]~\mathcal{P}[{:},i_2,\ldots,i_{N-1},{:}]
    \end{align*}
\end{defi}
\begin{theorem}\label{decN}
    The $\{2{:}N-1;N+1{:}2(N-1)\}$-TH tensor
    $ \mathsf{TH}(\mathcal{A})\in \mathbb{R}^{I_1\times I_2\times \cdots\times I_N\times I_2\times I_N-1} $,
    where
    $\widehat{\mathcal{A}}=(\Gamma,\cdots,\Gamma)_{2,\cdots,N-1}.{\mathcal{A}}
        \in \mathbb{R}^{I_{1} \times I_{2} \times \cdots \times I_{N-1} \times J}$, $\mathcal{X} \in \mathbb{R}^{J \times I_{2} \times \cdots \times I_{N}}$
    could be decomposed as follows:
    \[
        \mathsf{TH}(\mathcal{A})=({\boldsymbol{C}}^{\raisebox{1pt}{$-1$}},\cdots,{\boldsymbol{C}}^{\raisebox{1pt}{$-1$}},{\boldsymbol{C}}^{\raisebox{1pt}{${\sf T}$}},\cdots,{\boldsymbol{C}}^{\raisebox{1pt}{${\sf T}$}} )_{2{:}N-1;N+1:2(N-1)}.\mathcal{D}
    \]
    \[
        \mathcal{D}={\sf diag}_{\{2:N-1;N+1:2N-2\}}(\overline{\mathcal{A}}), \quad
        \mathcal{A}                              =(\overline{\boldsymbol{C}},\cdots,\overline{\boldsymbol{C}})_{2{:}N}.{\mathcal{A}}
    \]
\end{theorem}
\begin{proof}
    Without loss of generality, let
    $\mathcal{A}\in\mathbb{R}^{I_1\times I_2\times I_3 \times I_4}$ .So $\mathsf{TH}(\mathcal{A}) \in\mathbb{R}^{I_1\times I_2\times I_3 \times I_4\times I_2\times I_3}$
    will be $\{2,3;5,6\}-$TH tensor. since $\mathsf{TH}(\widehat{\mathcal{A}})$ is $\{3,6\}-$TH tensor, from lemma \ref{lemm_diagonal}, we have
    \begin{align}\label{eqtictic}
        \mathsf{TH}(\widehat{\mathcal{A}})=({\boldsymbol{C}}^{\raisebox{1pt}{$-1$}},{\boldsymbol{C}}^{\raisebox{1pt}{${\sf T}$}})_{3,6}.\overline{\mathcal{D}}, \quad \overline{\mathcal{D}}={\sf diag}_{\{3,6\}}(\mathcal{D}_1), \quad \mathcal{D}_1=({\boldsymbol{C}})_3.\mathsf{TH}(\widehat{\mathcal{A}})[{:},{:},{:},{:},{:},1]
    \end{align}
    So,
    \begin{align}\label{eqpp}
        \overline{\mathcal{D}}[i_1,i_2,i_3,i_4,i_5,i_6]=\delta_{i_{3},i_{6}}\mathcal{D}_1[i_1,i_2,i_3,i_4,i_5,1].
    \end{align}
    It's clear that $\mathcal{D}_1$ is a $\{2,5\}$-TH tensor and so can be diagonalized as follows:
    \begin{align}\label{DD}
        \mathcal{D}_1=\left(C^{-1},C^{\sf T}\right)_{\{2,5\}}.{\sf diag(\mathcal{D}_2)}
    \end{align}
    where
    \begin{align}\label{D2}
        \mathcal{D}_2=({\boldsymbol{C}})_2.\mathcal{D}_1[{:},{:},{:},{:},1,1]
    \end{align}
    But by replacing $\mathcal{D}_1$ from \eqref{DD} in \eqref{D2}, we have
    \begin{align}\label{D2AB}
        \mathcal{D}_2=({\boldsymbol{C}}, {\boldsymbol{C}})_{2,3}.\mathsf{TH}(\mathcal{A})[{:},{:},{:},{:},1,1]=({\boldsymbol{C}}, {\boldsymbol{C}})_{2,3}.\mathcal{A}=\overline{\mathcal{A}}
    \end{align}
    by replacing $\mathcal{D}_1$ from \eqref{DD} to \eqref{eqtictic} and substituting it with $\mathcal{D}_1$,
    \begin{align}
        \overline{\mathcal{D}}=
          & {\sf diag}_{\{3,6\}}(\mathcal{D}_1)={\sf diag}_{\{3,6\}}\left(({\boldsymbol{C}}^{\raisebox{1pt}{$-1$}},{\boldsymbol{C}}^{\raisebox{1pt}{${\sf T}$}})_{2,5}.{\sf diag}_{\{2,5\}}(\mathcal{D}_2)\right) \nonumber \\
        = & ({\boldsymbol{C}}^{\raisebox{1pt}{$-1$}},{\boldsymbol{C}}^{\raisebox{1pt}{${\sf T}$}})_{2,5}. \left( {\sf diag}_{\{3,6\}}\left({\sf diag}_{\{2,5\}}(\mathcal{D}_2)\right)\right)                      \nonumber \\
        = & ({\boldsymbol{C}}^{\raisebox{1pt}{$-1$}},{\boldsymbol{C}}^{\raisebox{1pt}{${\sf T}$}})_{2,5}. ({\sf diag}_{\{\{2,5\},\{3,6\}\}}(\mathcal{D}_2))                                                      \nonumber  \\
        = &
        ({\boldsymbol{C}}^{\raisebox{1pt}{$-1$}},{\boldsymbol{C}}^{\raisebox{1pt}{${\sf T}$}})_{2,5}.{\sf diag}_{\{\{2,5\},\{3,6\}\}}(\overline{\mathcal{A}})
    \end{align}
    which the last equation comes from \eqref{D2AB} and completes the proof.
\end{proof}

\begin{lemma}
    The $\star_c{}\text{-Product}$ $\mathcal{Y}=\mathcal{A}\star_c \mathcal{X}$ for $N$-order tensors $\mathcal{A}$ and $ \mathcal{X}$ could be handeled by cosine transform as follows:
    \[\mathcal{Y}=({\boldsymbol{C}}^{\raisebox{1pt}{$-1$}},\cdots,{\boldsymbol{C}}^{\raisebox{1pt}{$-1$}})_{2{:}N-1}\overline{\mathcal{Y}}\]
    where
    \begin{align*}
        \overline{\mathcal{Y}}  =\overline{\mathcal{A}}\odot_{2{:}N-1}\overline{\mathcal{X}},\quad
        \overline{\mathcal{X}}  =({\boldsymbol{C}},\cdots,{\boldsymbol{C}})_{2{:}N-1}.\mathcal{X},\quad
        \overline{\mathcal{A}}  =({\boldsymbol{C}},\cdots,{\boldsymbol{C}})_{2{:}N-1}.\mathcal{A}.
    \end{align*}
\end{lemma}
\begin{proof}
    By substituting the decomposition of $\mathsf{TH}(\widehat{\mathcal{A}})$ in Theorem \ref{decN}, in
    \begin{align*}
        \mathcal{Y}=\langle{\mathsf{TH}}(\widehat{\mathcal{A}}),\mathcal{X}\rangle_{N{:}2(N-1),2{:}N-1}
    \end{align*}
    we have
    \begin{align}\label{oh}
        \mathcal{Y} & =({\boldsymbol{C}}^{\raisebox{1pt}{$-1$}},\cdots,{\boldsymbol{C}}^{\raisebox{1pt}{$-1$}})_{2{:}N-1}.\left[
        \langle\mathcal{D},({\boldsymbol{C}},\cdots,{\boldsymbol{C}})_{2{:}N-1}.\mathcal{X}\rangle_{N{:}2(N-1),2{:}N-1}
        \right]                                                                                                                      \\
        \nonumber
                    & = ({\boldsymbol{C}}^{\raisebox{1pt}{$-1$}},\cdots,{\boldsymbol{C}}^{\raisebox{1pt}{$-1$}})_{2{:}2(N-1)}.\left[
        \langle\mathcal{D},\overline{\mathcal{X}}\rangle_{N{:}2(N-1),2{:}N-1}
        \right]
    \end{align}
    Now its clear that
    \begin{align*} \overline{\mathcal{Y}} & =\langle D,\overline{\mathcal{X}}\rangle_{N{:}2(N-1),2{:}N-1} \\
                                      & =\overline{\mathcal{A}}\odot_{2{:}N-1}\overline{\mathcal{X}}
    \end{align*}
    where the second equation comes from definition of $\mathcal{D}$ and completes the proof.
\end{proof}

The algorithm of computing $\star_c{}\text{-Product}$ for arbitrary order tensor is provided in Algorithm \ref{Alg2}.
\begin{algorithm}[h]
    \caption{$\star_c{}\text{-Product}$ for arbitrary order tensors}\label{Alg2}
    \begin{algorithmic}[1]
        \Require  $\mathcal{A}\in \mathbb{R}^{I_{1} \times I_{2} \times \cdots \times I_{N-1}\times J} $,
        $ \mathcal{X}\in \mathbb{R}^{J\times I_{2} \times  I_{3} \times \cdots \times I_{N}} $
        \Ensure $\mathcal{Y}\in \mathbb{R}^{ I_{1} \times \cdots \times I_{N}}$
        \State $ {\bar{\mathcal{A}}}=\left(C, \cdots, C \right)_{2:N-1}.\mathcal{A}$,
        $ {\bar{\mathcal{Y}}}=\left(C, \cdots, C\right)_{2:N-1}.\mathcal{Y} $
        \State
        ${\mathcal{\bar{Y}}}[:,\bar{i},:] = {\mathcal{\bar{A}}}[:,\bar{i},:]{\mathcal{\bar{X}}}[:,\bar{i},:],\quad \bar{i}=i_2,\cdots,i_{N-1},i_k=1,\cdots, I_k $
        \State ${\mathcal{Y}} = \left( C^{\sf T}, \cdots, C^{\sf T} \right)_{2:N-1}.{\overline{\mathcal{Y}}}$
    \end{algorithmic}
\end{algorithm}

Also we provide the SVD decomposition for any arbitrary $ N $-order tensor based on $\star_c$ product. For obtaining this decomposition we need to define transpose of N-order tensor based on $\star_c$ product.
$\mathcal{A}^{\sf T}$ is adjoint operator(Transpose) of $A$ according	to $\star_c$ operator if satisfies
\begin{align*}
    \left < \mathcal{A}\star_c \mathcal{X},\mathcal{Y}\right >=
    \left <  \mathcal{X},\mathcal{A}^{\sf T}\star_c\mathcal{Y}\right >
\end{align*}
for every tensors $\mathcal{X}, \mathcal{Y}.$ By some mathematical manupulation , one can see that this Transpose tensor is
$\mathcal{A}^{\sf T}\in \mathbb{R}^{I_N\times I_2\times \cdots \times I_{N-1}\times I_1}$ with the following elements
\begin{align}
    \mathcal{A}^{\sf T}[i_1,i_2,\cdots,i_{N-1},i_N]= \mathcal{A}[i_N,i_2,\cdots,i_{N-1},i_1]
\end{align}

\begin{theorem}:\label{TcSVDTh}
    \textbf{$\star_c{}\text{-SVD}$ for an $ n $-order tensor.}For each tensor $ \mathcal{A} \in \mathbb{R}^{I_1\times \cdots\times I_N} $ with arbitrary order $N$, the following decomposition named $\star_c{}\text{-SVD}$ exists
    \begin{align*}
        \mathcal{A} = \mathcal{U}\star_c \mathcal{S} \star_c \mathcal{V}^{\mathsf{T}}
    \end{align*}
    where
    $ \mathcal{U} \in \mathbb{R}^{I_1\times \cdots\times I_{N-1}\times I_N} $ and
    $ \mathcal{V} \in \mathbb{R}^{I_1\times \cdots\times I_{N-1}\times I_N} $
    are orthogonal tensors according to $\star_c{}\text{-Product}$ and $ \mathcal{S} \in \mathbb{R}^{I_N \times I_{2}\times \cdots \times I_{N-1} \times I_N} $ is an
    $ \{1,N\} $-diagonal tensor.
\end{theorem}
\begin{proof}:
    Let $ \overline{\mathcal{A}}=(\boldsymbol{C}_2,\cdots,\boldsymbol{C}_{N-1})_{2{:}N-1}.\mathcal{A} $, for every index set $ \bar{i}=i_2,\cdots,i_{N-1} $ consider the SVD of $ \overline{\mathcal{A}}[{:},\bar{i},{:}] $ matrix as follows{:}
    \begin{align}\label{eq_plus_1}
        \overline{\mathcal{A}}[{:},\bar{i},{:}]=
        \overline{U}_{\bar{i}}\overline{S}_{\bar{i}}\overline{V}_{\bar{i}}^{\mathsf{T}}, \quad
        \overline{\boldsymbol{U}}_{\bar{i}} \in \mathbb{R}^{I_1\times I_1}, \quad \overline{\boldsymbol{S}}_{\bar{i}} \in \mathbb{R}^{I_1\times I_N}, \quad
        \overline{\boldsymbol{V}}_{\bar{i}} \in \mathbb{R}^{I_N\times I_N}.
    \end{align}
    if we construct the tensors $ \overline{{\mathcal{U}}}\in\mathbb{R}^{I_1\times I_2 \times \cdots\times I_{N-1}\times I_1} $, $\overline{{\mathcal{S}}}\in \mathbb{R}^{I_1\times I_2 \times \cdots I_{N-1}\times I_N}$ and
    $ \overline{\mathcal{V}}\in\mathbb{R}^{I_N\times I_2 \times \cdots\times I_{N-1}\times I_N} $, as follows{:}
    \begin{align*}
        \overline{\mathcal{U}}[{:},\bar{i},{:}]=\overline{U}_{\bar{i}},\quad
        \overline{\mathcal{S}}[{:},\bar{i},{:}]=\overline{S}_{\bar{i}},\quad
        \overline{\mathcal{V}}[{:},\bar{i},{:}]=\overline{{V}}_{\bar{i}},\quad
        \bar{i} = i_2,\cdots, i_{N-1}.
    \end{align*}
    We claim that
    \begin{align*}
        \mathcal{P}=\mathcal{U}\star_c \mathcal{S} \star_c \mathcal{V}^{\mathsf{T}}
    \end{align*}
    where
    \begin{align*}   
    \mathcal{U}=(\boldsymbol{C}_2^{-1},\cdots,\boldsymbol{C}_{N-1}^{-1})_{2{:}N-1}.\overline{\mathcal{U}}, \quad
    \mathcal{V}=(\boldsymbol{C}_2^{-1},\cdots,\boldsymbol{C}_{N-1}^{-1})_{2{:}N-1}.\overline{\mathcal{V}}, \quad
    \mathcal{S}=(\boldsymbol{C}_2^{-1},\cdots,\boldsymbol{C}_{N-1}^{-1})_{2{:}N-1}.\overline{\mathcal{S}}
    \end{align*}
    is equal to $\mathcal{A}$. For showing this
    let $\boldsymbol{\mathcal{Q}}=\mathcal{S} \star_c \mathcal{V}^{\mathsf{T}}$, then $\mathcal{P}=\mathcal{U}\star_c \boldsymbol{\mathcal{Q}}$.
    By definition of $\star_c{}\text{-Product}$ we have
    \begin{align*}
        \overline{\mathcal{P}}[{:},i,{:}]              & =\overline{\mathcal{U}}[{:},i,{:}]\overline{\boldsymbol{\mathcal{Q}}}[{:},i,{:}] \\
        \overline{\boldsymbol{\mathcal{Q}}}[{:},i,{:}] & =\overline{\mathcal{S}}[{:},i,{:}]\overline{\mathcal{V}}^{\mathsf{T}}[{:},i,{:}]
    \end{align*}
    So,
    $\overline{\mathcal{P}}[{:},\bar{i},{:}]=\overline{\mathcal{U}}[{:},\bar{i},{:}]\overline{\mathcal{S}}[{:},\bar{i},{:}]\overline{\mathcal{V}}^{\mathsf{T}}[{:},\bar{i},{:}]=\overline{\mathcal{A}}[{:},\bar{i},{:}]$
    therefore $\overline{\mathcal{P}}=\overline{\mathcal{A}}$ and so $\mathcal{P}=\mathcal{A}$.
\end{proof}
The algorithm \ref{Algf} provides the process of $\star_c{}\text{-SVD}$ for an arbitrary order tensor.
Also, like $t-SVD$ the truncated version of $\star_c{}\text{-SVD}$ for arbitrary $r \leq \min{\{I_1,I_N\}}$ could be obtained
as
follows
\[
    \mathcal{A}\approx \mathcal{A}_r=\mathcal{U}_r\star_c\mathcal{S}_r\star_c\mathcal{V}_r^{\sf T}
\]
where
\begin{align}\label{man}
    \overline{\mathcal{U}_r}[:,\bar{i},:]=\overline{U}_{\bar{i}}[:,1:r], \quad \overline{\mathcal{S}_r}(:,\bar{i},:)=\overline{S}_{\bar{i}}[1:r,1:r]\quad \overline{\mathcal{V}_r}(:,\bar{i},:)=\overline{V}_{\bar{i}}[:,1:r]^{\sf T}
\end{align}
from SVD of $\overline{\mathcal{A}}[:,\bar{i},:]$ in \eqref{eq_plus_1}.
This truncation could be used for different tasks like compression and denoising. 

The mentioned truncation is applied truncated SVD in each slices of data in cosine space. In the following we show that further this truncation another truncation that removes some slices (corresponding to some specific frequencies) in cosine space also could be defined for this decomposition.
As one know  for arbitrary vector $a$, $\bar{a}=Ca$, denotes the coefficient of cosine transform. Here the $\bar{a}_i$ with small and large indices correspond to low and high frequencies. So, by removing the $a_i$ with large indices, we have one high pass filter that could be used as denoising and also compression of data in $a$. So in $\star_c$-SVD, of order-3 tensor $\mathcal{A}$, the slice's $\overline{A}[:,i,:]$ for small and large indices, contain low and high pass coefficients, respectively. Therefor if for some specific index like $l$, we set $\overline{A}[:,i,:]=0$ for $i>l$, we  have high pass filter. Now if for $i\leq l$, we use the truncated filter like \eqref{eq_plus_1}, we have a double filter decomposition.
This double filtering uses the benefits of filtering on svd and frequency space, simultaneously and we expect that has better results in comparison filers like \eqref{man}.
This double truncation based on $\star_c-SVD$ for arbitrary order can be seen in  algorithm \ref{Algf}. 

\begin{algorithm}
    \caption{Double filtering by $\star_c{}\text{-SVD}$}\label{Algf}
    \begin{algorithmic}[1]
        \Require
        $ \mathcal{A}\in \mathbb{R}^{I_1\times \cdots\times I_N} $,
        $\ell$ truncation index for cosine space,
        $k$ truncation index for SVD space,
        \Ensure truncated $\star_c{}\text{-SVD}$ ($\mathcal{A}_k^l=\mathcal{U}_k^l\star_c\mathcal{S}_k^l\star_c\mathcal{V}_k^l
            \approx \mathcal{A}$)
        \State Compute $\overline{\mathcal{A}}=\left(C,\cdots, C\right)_{2:N-1}.\mathcal{A}$
        \For{ all $\bar{i}=i_2,\cdots,i_{N-1}$, $\overline{\mathcal{A}}[:,\bar{i},:]$}
        \If { $i_s<\ell$ for all $s=2,\cdots,N-1$}
        \State Compute
        $\overline{\mathcal{A}}[:,\bar{i},:]=\overline{\mathcal{U}}_{\bar{i}}[:,1{:}k]\overline{\mathcal{S}}_{\bar{i}}[1{:}k,1{:}k]\overline{\mathcal{V}}_{\bar{i}}[:,1{:}k]^T$
        \Else
        \State Let $\overline{\mathcal{A}}[:,\bar{i},:]=0$ i.e (
        $\overline{\mathcal{U}}_{\bar{i}}[:,1{:}k]=0,
            \overline{\mathcal{S}}_{\bar{i}}[1{:}k,1{:}k]=0,
            \overline{\mathcal{V}}_{\bar{i}}[:,1{:}k]=0$
        )
        \EndIf
        \EndFor
        \State $\mathcal{U}_k^l=\left(C^{-1},\cdots, C^{-1}\right)_{2:N-1}.\overline{\mathcal{U}}$
        \State $\mathcal{S}_k^l=\left(C^{-1},\cdots, C^{-1}\right)_{2:N-1}.\overline{\mathcal{S}}$
        \State $\mathcal{V}_k^l=\left(C^{-1},\cdots, C^{-1}\right)_{2:N-1}.\overline{\mathcal{V}}$
    \end{algorithmic}
\end{algorithm}

\section{Experimental results}\label{Exp}

In this section, we present experiments on both synthetic random tensors and well-known datasets, as described in Table \ref{table_databaseName}, to demonstrate the efficacy of the proposed $\star_c{}\text{-SVD}$ method. The last dimension in Table \ref{table_databaseName} represents the number of samples for each dataset. For example, if dataset $\mathcal{A} \in R^{I_1,I_2,\dots,I_N}$, then $I_N$ denotes the number of samples in $\mathcal{A}$. In these experiments, we compare the performance of our $\star_c{}\text{-SVD}$ method to that of t-SVD in compression, clustering, and classification applications.
\begin{table}[h]
    \caption{Characters of different datasets}\label{table_databaseName}
    \resizebox{0.85\linewidth}{!}{
        \begin{tabular}{|c|c|ccc|}
            \hline
            Dataset Name                                    & Size                                & Compression & Clustering & Classification \\
            \hline
            \begin{tabular}[c]{@{}c@{}}
                Brain MRI \\
                (Brain Tumor Detection)  \cite{datasetBrainMRI}
            \end{tabular} & $253\times  240\times 240$          & Yes         & No         & Yes                                              \\
            \hline
            CBCL-face                                       & $2414\times 192\times 168$          & Yes         & No         & Yes            \\
            \hline
            Coil-100  \cite{nene1988columbia}               & $7200\times 128\times 128$          & Yes         & No         & Yes            \\
            \hline
            Digit Recognizer (MNIST)                        & $60000\times 28\times 28 $          & Yes         & No         & Yes            \\
            \hline
            Yale                                            & $2414\times 192\times 168 $         & Yes         & No         & Yes            \\
            \hline
            Cifar                                           & $60000\times 32\times 32\times 3$   & Yes         & No         & Yes            \\
            \hline
            3D-MNIST                                        & $10000\times 16\times 16\times 16$  & Yes         & No         & Yes            \\
            \hline
            Bonsai                                          & $256\times 256 \times 256$          & Yes         & No         & No             \\
            \hline
            Engine                                          & $256 \times 256 \times 128$         & Yes         & No         & No             \\
            \hline
            Foot                                            & $256\times 256 \times 256$          & Yes         & No         & No             \\
            \hline
            Skull                                           & $256\times 256 \times 256$          & Yes         & No         & No             \\
            \hline
            StatueLeg                                       & $341 \times 341 \times 93$          & Yes         & No         & No             \\
            \hline
            SyntheticA                                      & $100\times 40\times 100$            & Yes         & No         & No             \\
            \hline
            SyntheticB                                      & $100\times 100\times 100$           & Yes         & No         & No             \\
            \hline
            SyntheticC                                      & $100\times 500\times 100$           & Yes         & No         & No             \\
            \hline
            SyntheticD                                      & $100\times 500\times 500\times 100$ & Yes         & No         & No             \\
            \hline
            Letters                                         & $20000\times 16$                    & No          & Yes        & No             \\
            \hline
            PIE pose                                        & $2856\times 32 \times 32$           & No          & Yes        & No             \\
            \hline
            PenDigits                                       & $10992\times 16$                    & No          & Yes        & No             \\
            \hline
            USPS                                            & $11000\times 16\times 16$           & No          & Yes        & No             \\
            \hline
        \end{tabular}
    }
\end{table}

\subsection{Compression}
As the first application, we use the proposed tensor decomposition as a compression method and compare tits results with the $t$-SVD.
In compression, only the decomposed factors of the main data are stored, and the data is reconstructed when it is called. In $\star_c$-SVd and $t$-SVd the factors could be stored in original or frequency spaces as follows:
\begin{itemize}
    \item \textbf{Storage in the frequency domain(SFD):}Due to properties of these decomposition's we show storing factors in frequency space is better than storing in the main space in terms of storage space and computational complexity. For simplicity and without lose of generality consider 3-order tensor $\mathcal{A}\in \mathbb{R}^{I_1\times I_2\times I_3}$. In this case the storage of factors of double filtered $\star_c$-SVD(l,k) is from order $(I_1+I_3)lk$ and the operations for reconstruction of $\mathcal{U}_k^l\star_c \mathcal{S}_k^l\star_c \mathcal{U}_k^{l^{\sf T}}$ is from order $I_1I_3kl+I_1I_3I_2\log(I_2)$. Because the factors of the $t$-SVD are complex in the frequency space, the storage space and operations required to reconstruct the original data are twice more than $\star_c$-SVD method.
    \item \textbf{Storage in the main domain(SMD):} If one saves the factors of $\star_c$-SVD in the main space, the storage and operations for reconstruction are from orders $(I_1+I_3)I_2k$ and $I_1I_3kl+2I_1I_3I_2\log(I_2)$, respectively, which is more than the first case. In this case the storage and operations for reconstruction for $t$-SVD are from orders $(I_1+I_3)I_2k$ and $2(I_1I_3kl+2I_1I_3I_2\log(I_2))$, respectively.
\end{itemize}
Its clear that storing the factors in the frequency domain is better than the last case. Also $\star_c$-SVD has less storage and reconstruction computational complexity  in comparison with $t$-SVD, in both cases. Also, mentioned complexities are summarized in Table \ref{s}.

\begin{table}[!htbp] \centering
\caption{ Storage Vs operation for reconstruction in Two approaches(Saving in main (SMD) or frequency (SFD) domains) for $t$-SVD (l,k) and $\star_c$-SVD(l,k)}
\label{s}
\begin{tabular}{@{\extracolsep{5pt}} llrrr} 
\\[-1.8ex]\hline 
\hline \\[-1.8ex] 
\multicolumn{1}{c}{Saving Approach} & \multicolumn{1}{c}{Tensor decomposition} & \multicolumn{1}{c}{Storage} & \multicolumn{1}{c}{Operations for reconstruction}  \\
\hline \\[-1.8ex] 
\multirow{ 2 }{*}{ SFD }  &  $t$-SVD  &  $2(I_1+I_3)lk$ &  2($I_1I_3kl+I_1I_3I_2\log(I_2)$)  \\
~~~~~  &  $\star_c$-SVD  &  $(I_1+I_3)lk$ &  $I_1I_3kl+I_1I_3I_2\log(I_2)$ \\
\hline \\[-1.8ex] 
\multirow{ 2 }{*}{ SMD }  &  $t$-SVD  &  $(I_1+I_3)I_2k$  &  $2(I_1I_3kl+2I_1I_3I_2\log(I_2))$\\
~~~~~  &  $\star_c$-SVD  &  $(I_1+I_3)I_2k$   &  $I_1I_3kl+2I_1I_3I_2\log(I_2)$  \\
\hline \\[-1.8ex] 
\end{tabular}
\end{table}

In the following for compression we considered some experiments for compression. For simplicity we did not used double filtering(removing high frequencies did not used). So in the following $k$ denotes the truncation of slices of data in frequency domain.
Table  \ref{TaApprox} reports running time  and Frobenius norm error between exact synthetic data and reconstructed data by $t$-SVD and $\star_c{}\text{-SVD}$ for different $k$.It is obvious that $\star_c{}\text{-SVD}$ has a better approximation and running time.
\begin{table}[h!]
    \renewcommand{\arraystretch}{2}
    \caption{\footnotesize{Approximation data by truncated t-SVD and truncated $\star_c{}\text{-SVD}$ approach}\label{TaApprox}}
    \resizebox{1\linewidth}{!}{
        \begin{tabular}{|c|c|c||ccccccccccccc|}
            \hline
            Dataset                     & \multicolumn{2}{c}{method}              & \multicolumn{13}{c|}{k}                                                                                                                                                                                                                    \\
            \hline
            \multirow{5}{*}{SyntheticA} & \multicolumn{2}{c}{k}                   & $ 5 $                   & $ 6 $        & $ 7 $        & $ 8 $        & $ 9 $          & $ 10 $       & $ 11 $       & $ 12 $         & $ 15 $         & $ 20 $         & $ 25 $         & $ 30 $         & $ 35$                           \\\hline\hline
                                        & \multirow{2}{*}{$\star_c{}\text{-SVD}$} & $ \|.\|_F $             & $ 150.6799 $ & $ 145.8192 $ & $ 141.0347 $ & $ 136.3285 $   & $ 131.7132 $ & $ 127.1617 $ & $ 122.6856 $   & $ 118.2726 $   & $ 105.3673 $   & $ 84.8171 $    & $ 65.1434 $    & $ 45.9772 $    & $ 26.1436 $    \\
                                        &                                         & Time                    & $ 0.0080 $   & $ 0.0088 $   & $ 0.0089 $   & $ 0.0089 $     & $ 0.0092 $   & $ 0.0095 $   & $ 0.0095 $     & $ 0.0101 $     & $ 0.0108 $     & $ 0.0125 $     & $ 0.0142 $     & $ 0.0162 $     & $ 0.0176$      \\ \cline{2-16}
                                        & \multirow{2}{*}{t-SVD}                  & $ \|.\|_F $             & $ 151.0204 $ & $ 146.1793 $ & $ 141.4365 $ & $ 136.7594 $   & $ 132.1529 $ & $ 127.6142 $ & $ 123.1437 $   & $ 118.7473 $   & $ 105.8626 $   & $ 85.2859 $    & $ 65.6273 $    & $ 46.4386 $    & $ 26.4272$     \\
                                        &                                         & Time                    & $ 0.0107 $   & $ 0.0109 $   & $ 0.0118 $   & $ 0.0132 $     & $ 0.0132 $   & $ 0.01383 $  & $ 0.0146 $     & $ 0.0162 $     & $ 0.0185 $     & $ 0.0226 $     & $ 0.0278 $     & $ 0.0327 $     & $ 0.0392$      \\\hline
            \multirow{5}{*}{SyntheticB} & \multicolumn{2}{c}{k}                   & $ 5 $                   & $ 10 $       & $ 15 $       & $ 20 $       & $ 25 $         & $ 35 $       & $ 45 $       & $ 55 $         & $ 65 $         & $ 75 $         & $ 85 $         & $ 95 $         & $ 100$                          \\\hline\hline
                                        & \multirow{2}{*}{$\star_c{}\text{-SVD}$} & $ \|.\|_F $             & $ 267.5346 $ & $ 243.5508 $ & $ 221.3552 $ & $ 200.5162 $   & $ 180.8339 $ & $ 144.4158 $ & $ 111.6186 $   & $ 82.3923 $    & $ 56.6329 $    & $ 34.4405 $    & $ 16.3730 $    & $ 3.6567 $     & $ 0.2511$      \\
                                        &                                         & Time                    & $ 0.0124 $   & $ 0.0149 $   & $ 0.0164 $   & $ 0.0181 $     & $ 0.0205 $   & $ 0.0260 $   & $ 0.0299 $     & $ 0.0323 $     & $ 0.0430 $     & $ 0.0394 $     & $ 0.0471 $     & $ 0.0561 $     & $ 0.0562$      \\\cline{2-16}
                                        & \multirow{2}{*}{t-SVD}                  & $ \|.\|_F $             & $ 267.7304 $ & $ 243.8353 $ & $ 221.6783 $ & $ 200.8804 $   & $ 181.2331 $ & $ 144.8921 $ & $ 112.1286 $   & $ 82.9069 $    & $ 57.1755 $    & $ 34.9470 $    & $ 16.8801 $    & $ 3.8276 $     & $ 0.2875$      \\
                                        &                                         & Time                    & $ 0.0229 $   & $ 0.0282 $   & $ 0.0362 $   & $ 0.0341 $     & $ 0.0393 $   & $ 0.0509 $   & $ 0.0624 $     & $ 0.0680 $     & $ 0.0795 $     & $ 0.0936 $     & $ 0.0994 $     & $ 0.1129 $     & $ 0.1150$      \\\hline
            \multirow{5}{*}{SyntheticC} & \multicolumn{2}{c}{k}                   & $ 10 $                  & $ 15 $       & $ 30 $       & $ 50 $       & $ 80 $         & $ 100 $      & $ 150 $      & $ 200 $        & $ 250 $        & $ 300 $        & $ 350 $        & $ 450 $        & $ 500$                          \\\hline\hline
                                        & \multirow{2}{*}{$\star_c{}\text{-SVD}$} & $ \|.\|_F $             & $ 589.0791 $ & $ 559.4948 $ & $ 474.6183 $ & $ 366.0083 $   & $ 199.2323 $ & $ 36.4333 $  & $ 4.1076e-12 $ & $ 4.1076e-12 $ & $ 4.1076e-12 $ & $ 4.1076e-12 $ & $ 4.1076e-12 $ & $ 4.1076e-12 $ & $ 4.1076e-12 $ \\
                                        &                                         & Time                    & $ 0.0506 $   & $ 0.0543 $   & $ 0.0653 $   & $ 0.08732312 $ & $ 0.1087 $   & $ 0.1273 $   & $ 0.1274 $     & $ 0.1232 $     & $ 0.1219 $     & $ 0.1221 $     & $ 0.1223 $     & $ 0.1221 $     & $ 0.1254 $     \\\cline{2-16}
                                        & \multirow{2}{*}{t-SVD}                  & $ \|.\|_F $             & $ 589.3193 $ & $ 559.7748 $ & $ 474.9585 $ & $ 366.3456 $   & $ 199.5488 $ & $ 36.7452 $  & $ 3.5690e-11 $ & $ 3.5690e-11 $ & $ 3.5690e-11 $ & $ 3.5690e-11 $ & $ 3.5690e-11 $ & $ 3.5690e-11 $ & $ 3.5690e-11 $ \\
                                        &                                         & Time                    & $ 0.1117 $   & $ 0.1207 $   & $ 0.1553 $   & $ 0.1987 $     & $ 0.2686 $   & $ 0.3150 $   & $ 0.3180 $     & $ 0.3159 $     & $ 0.3180 $     & $ 0.3149 $     & $ 0.3157 $     & $ 0.3168 $     & $ 0.3174 $     \\\hline

            \multirow{5}{*}{SyntheticD} & \multicolumn{2}{c}{k}                   & $ 10 $                  & $ 15 $       & $ 30 $       & $ 50 $       & $ 80 $         & $ 100 $      & $ 150 $      & $ 200 $        & $ 250 $        & $ 300 $        & $ 350 $        & $ 450 $        & $ 500$                          \\\hline\hline
                                        & \multirow{2}{*}{$\star_c{}\text{-SVD}$} & $ \|.\|_F $             & $ 589.0791 $ & $ 559.4948 $ & $ 474.6183 $ & $ 366.0083 $   & $ 199.2323 $ & $ 36.4333 $  & $ 4.1076e-12 $ & $ 4.1076e-12 $ & $ 4.1076e-12 $ & $ 4.1076e-12 $ & $ 4.1076e-12 $ & $ 4.1076e-12 $ & $ 4.1076e-12 $ \\
                                        &                                         & Time                    & $ 0.0506 $   & $ 0.0543 $   & $ 0.0653 $   & $ 0.08732312 $ & $ 0.1087 $   & $ 0.1273 $   & $ 0.1274 $     & $ 0.1232 $     & $ 0.1219 $     & $ 0.1221 $     & $ 0.1223 $     & $ 0.1221 $     & $ 0.1254 $     \\\cline{2-16}
                                        & \multirow{2}{*}{t-SVD}                  & $ \|.\|_F $             & $ 589.3193 $ & $ 559.7748 $ & $ 474.9585 $ & $ 366.3456 $   & $ 199.5488 $ & $ 36.7452 $  & $ 3.5690e-11 $ & $ 3.5690e-11 $ & $ 3.5690e-11 $ & $ 3.5690e-11 $ & $ 3.5690e-11 $ & $ 3.5690e-11 $ & $ 3.5690e-11 $ \\
                                        &                                         & Time                    & $ 0.1117 $   & $ 0.1207 $   & $ 0.1553 $   & $ 0.1987 $     & $ 0.2686 $   & $ 0.3150 $   & $ 0.3180 $     & $ 0.3159 $     & $ 0.3180 $     & $ 0.3149 $     & $ 0.3157 $     & $ 0.3168 $     & $ 0.3174 $     \\\hline
        \end{tabular}
    }
\end{table}
The experiments shows the the proposed method  gives the same approximation in error with more less time, in comparison with $t$-SVD for different $k$. Since this table is based on $k$, it can not show the relation between storage and quality of the approximations, for two methods. So in the next experiment , we reports in \ref{figrun2} and \ref{figrun2r} the quality of reconstruction (PSNR) vs Storage(Bytes) for two methods applied on real data in 
Table \ref{table_databaseName} for different $k$.  
\begin{figure}[ht]
    \centering
    \begin{subfigure}[t]{0.5\textwidth}
        \centering
        \includegraphics[width=0.65\textwidth]{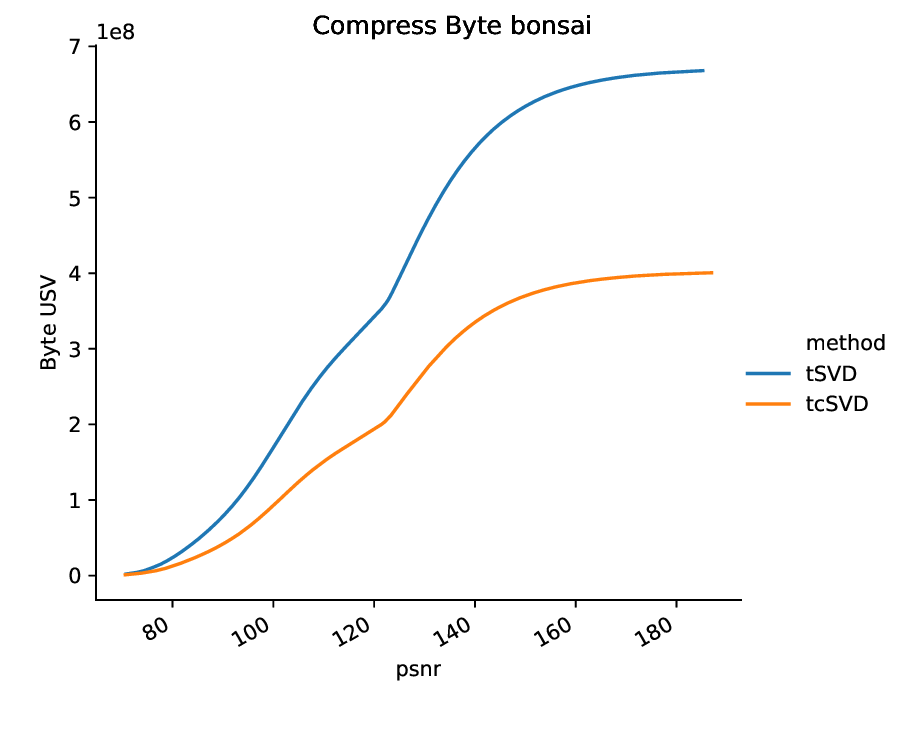}
    \end{subfigure}%
    \begin{subfigure}[t]{0.5\textwidth}
        \centering
        \includegraphics[width=0.65\textwidth]{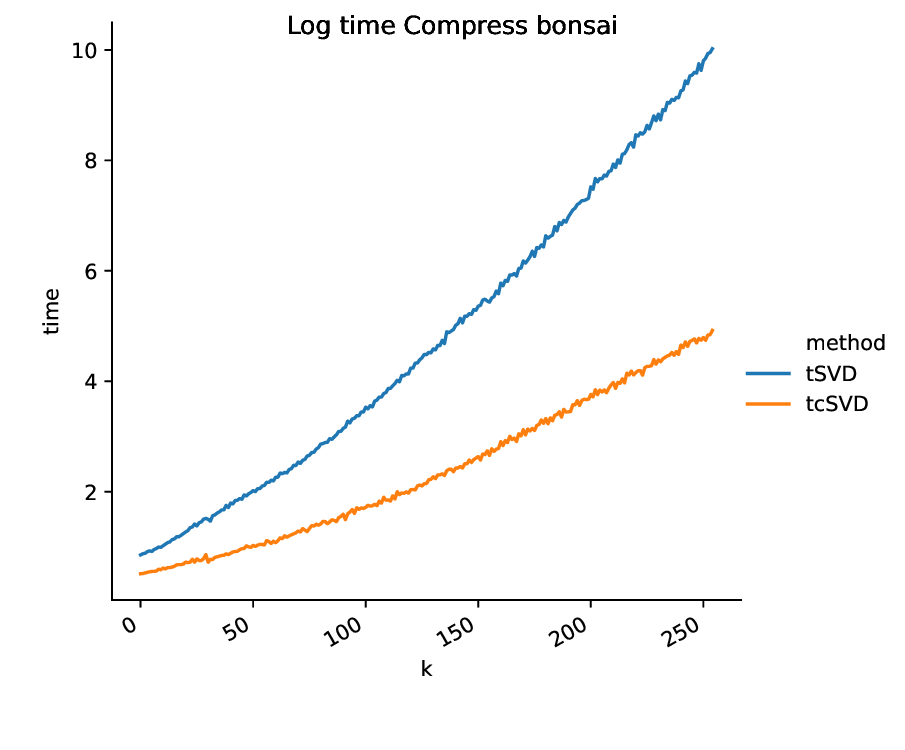}
    \end{subfigure}%
    \newline
    \begin{subfigure}[t]{0.5\textwidth}
        \centering
        \includegraphics[width=0.65\textwidth]{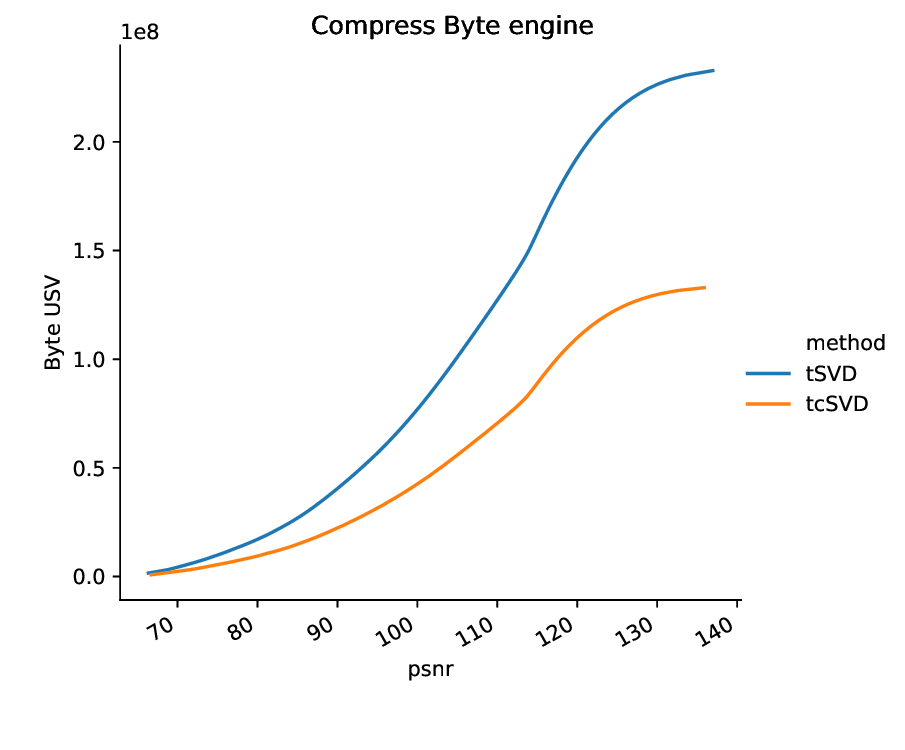}
    \end{subfigure}%
    \begin{subfigure}[t]{0.5\textwidth}
        \centering
        \includegraphics[width=0.65\textwidth]{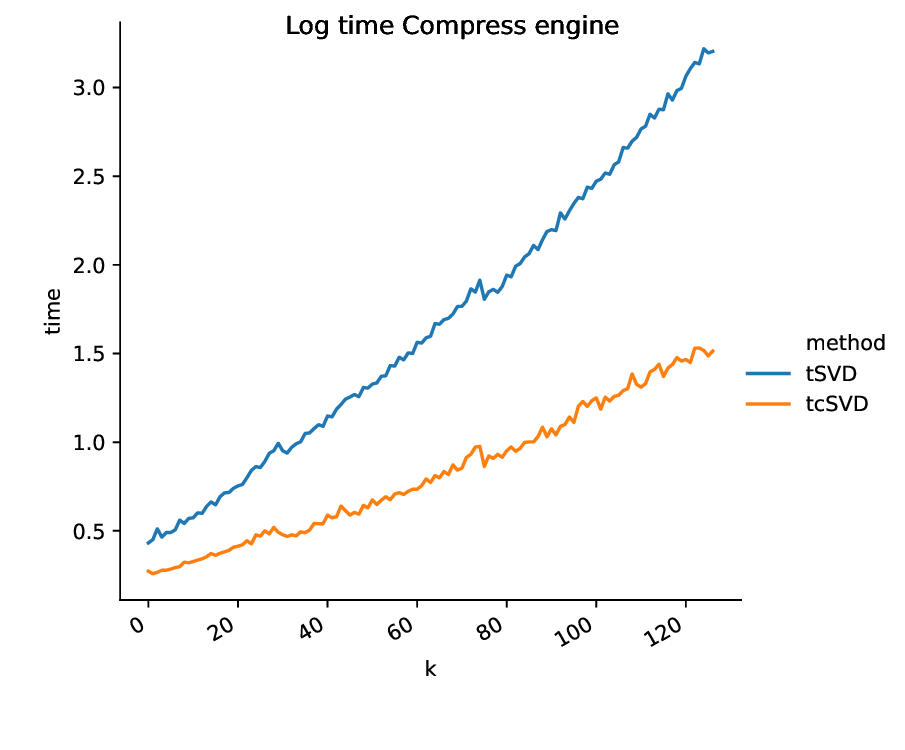}
    \end{subfigure}%
    \newline
    \begin{subfigure}[t]{0.5\textwidth}
        \centering
        \includegraphics[width=0.65\textwidth]{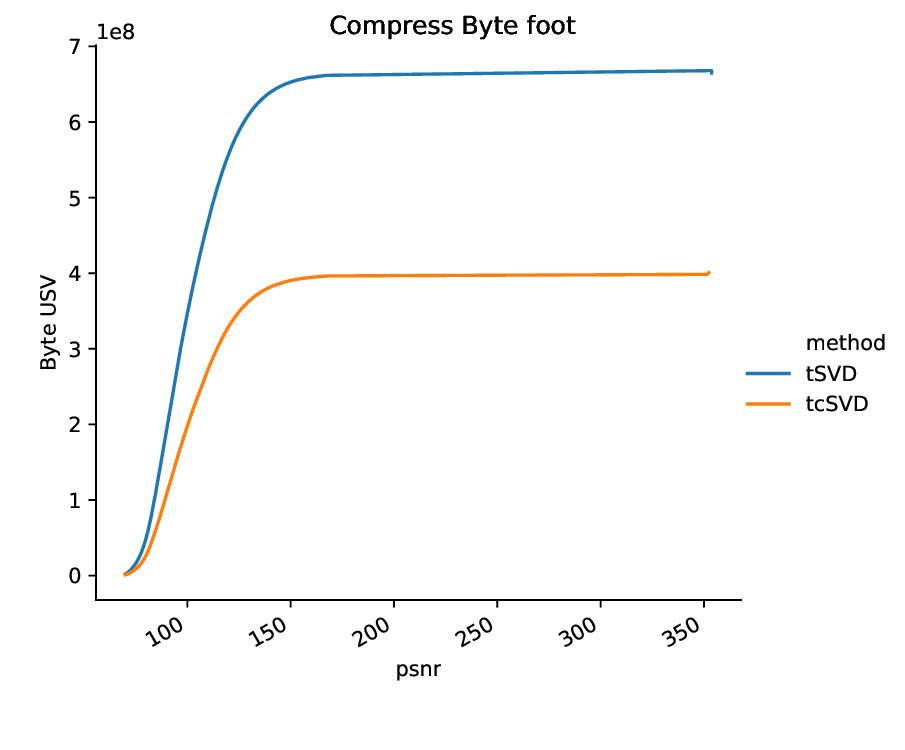}
    \end{subfigure}%
    \begin{subfigure}[t]{0.5\textwidth}
        \centering
        \includegraphics[width=0.65\textwidth]{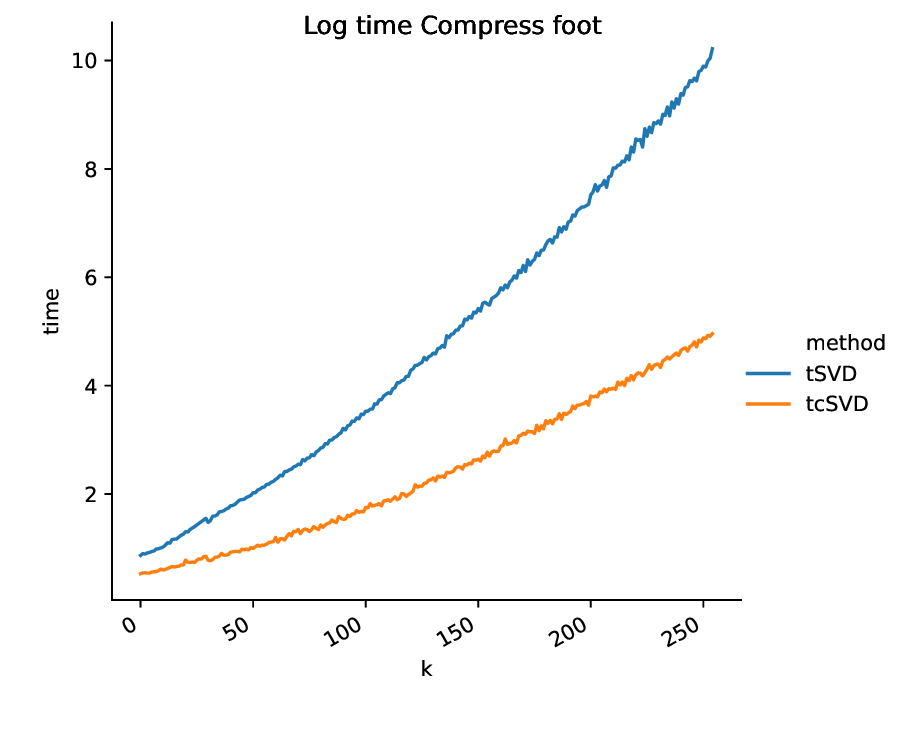}
    \end{subfigure}%
\newline
    \begin{subfigure}[t]{0.5\textwidth}
        \centering
        \includegraphics[width=0.65\textwidth]{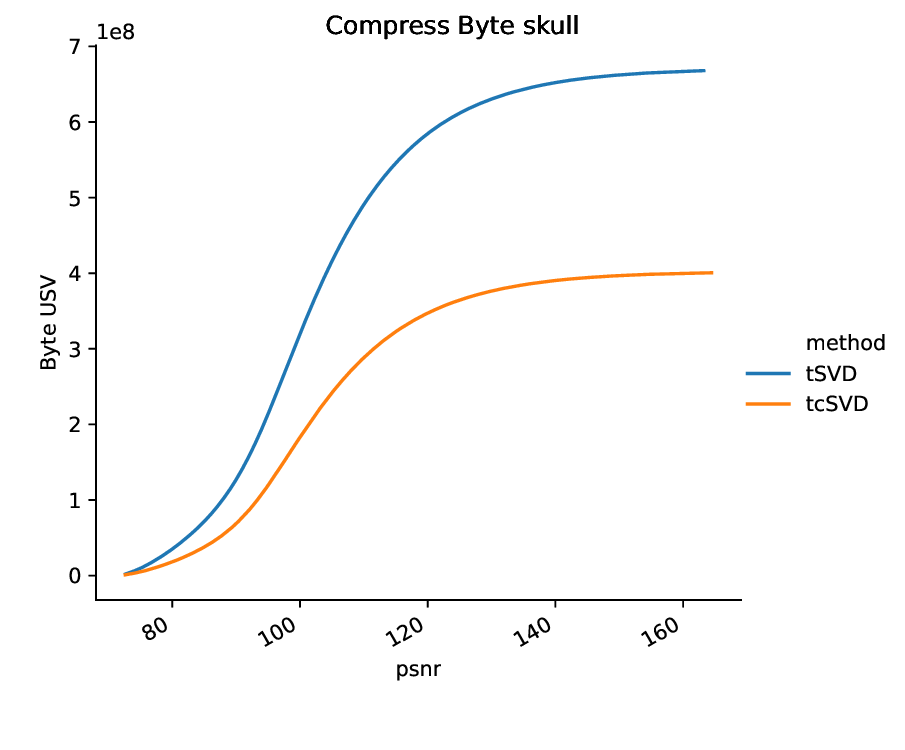}
    \end{subfigure}%
    \begin{subfigure}[t]{0.5\textwidth}
        \centering
        \includegraphics[width=0.65\textwidth]{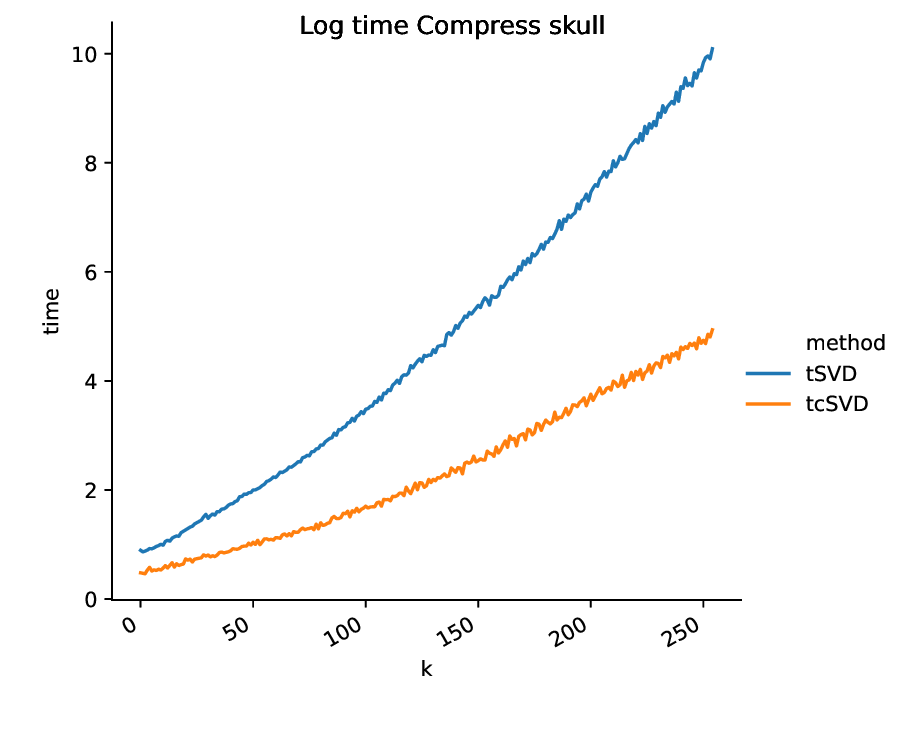}
    \end{subfigure}%
    \newline
    \begin{subfigure}[t]{0.5\textwidth}
        \centering
        \includegraphics[width=0.65\textwidth]{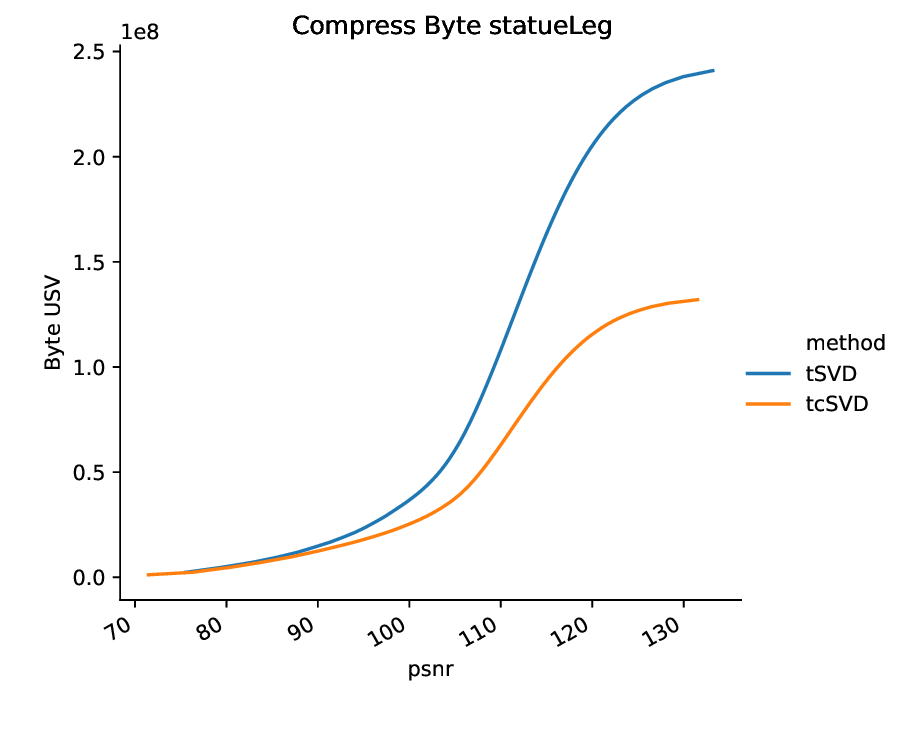}
    \end{subfigure}%
    \begin{subfigure}[t]{0.5\textwidth}
        \centering
        \includegraphics[width=0.65\textwidth]{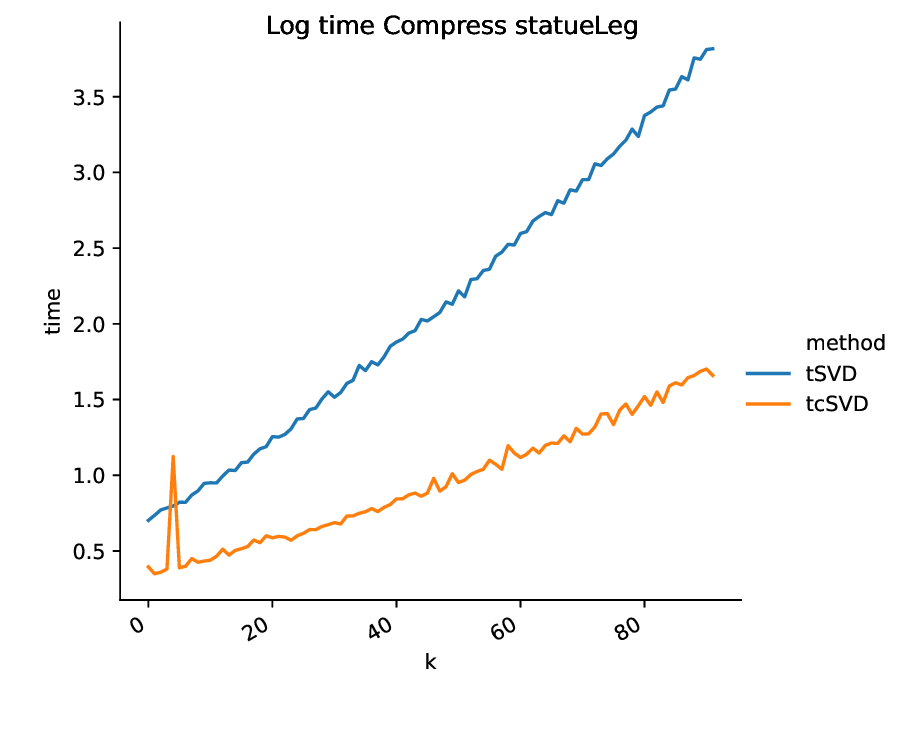}
    \end{subfigure}%
    \caption{(PSNR Vs storage) and ($k$ Vs Time) for $t$-SVD ans $\star_c$-SVD, on  some data from Table  \ref{table_databaseName}}\label{figrun2r}
\end{figure}
All experiments show the quality of the proposed $\star_c$-SVD over $t$-SVD method in compression of data.

\subsection{$ \star_c{}\text{-SVD} $ as Feature extraction method}
In this experiment, we first extracted new features using both the t-SVD and $\star_c{}\text{-SVD}$ approaches, and then we performed clustering and classification methods on the resulting data.
For Clustering method, we compute $\star_c{}\text{-SVD}$ ($\mathcal{A} $,$\mathbf{k_1}$), $\mathbf{k_1}:=(k_1,\cdots,k_{N-1})$ of the dataset $ \mathcal{A} \in R^{I_1\times I_2\times \dots \times I_N} $ to obtain data with reduced features $ \mathcal{A}_{k1} =(\mathcal{V}_{k1} \star_c\mathcal{A})^{\sf T} $. we used Normalized Mutual Information (NMI) criterion to compare the result of kmeans method on reduced data obtained by  t-SVD and $\star_c{}\text{-SVD}$ approaches.
You can see the details of k-means clustering on different datasets in Table \ref{Taclustek-means}.

\begin{figure}[ht]
	\begin{subfigure}[t]{0.5\textwidth}
		\centering
		\includegraphics[width=0.65\textwidth]{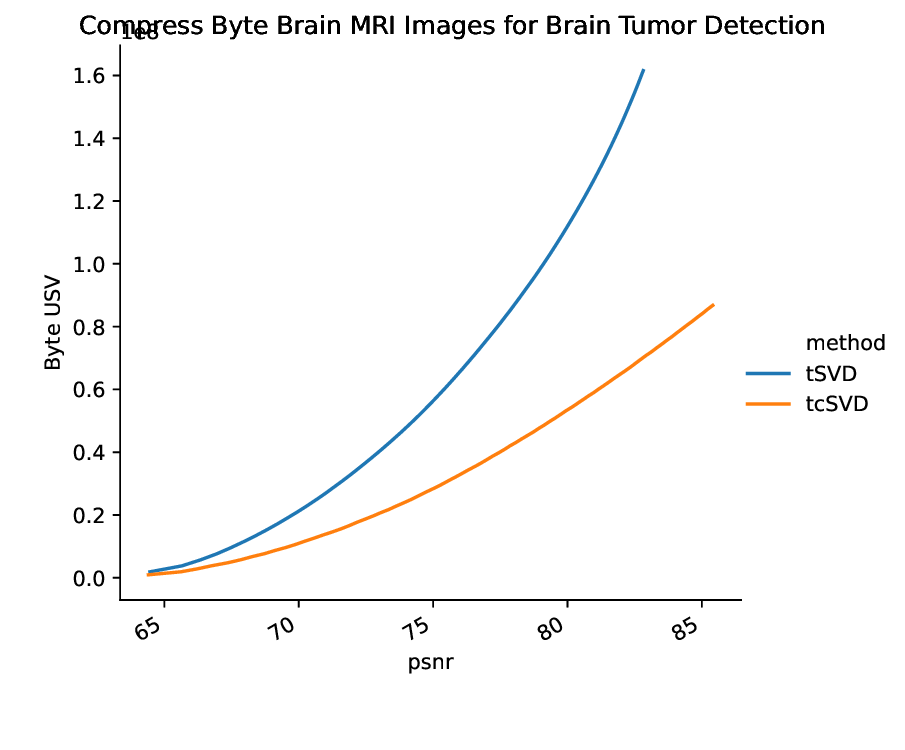}
	\end{subfigure}%
	\begin{subfigure}[t]{0.5\textwidth}
		\centering
	\includegraphics[width=0.65\textwidth]{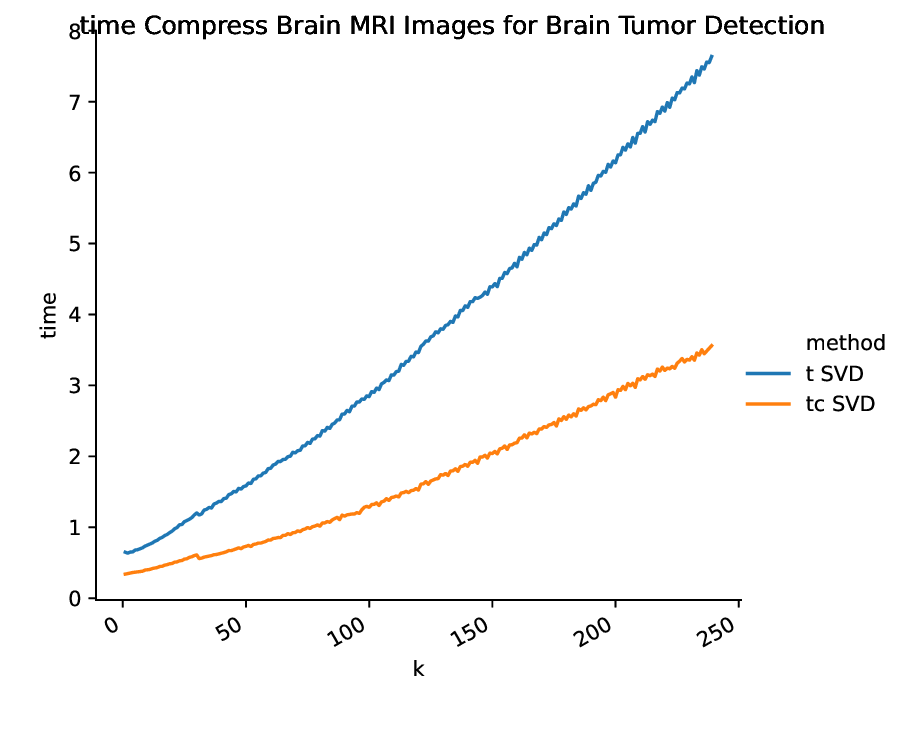}
	\end{subfigure}%
	\newline
	\begin{subfigure}[t]{0.5\textwidth}
		\centering
		\includegraphics[width=0.65\textwidth]{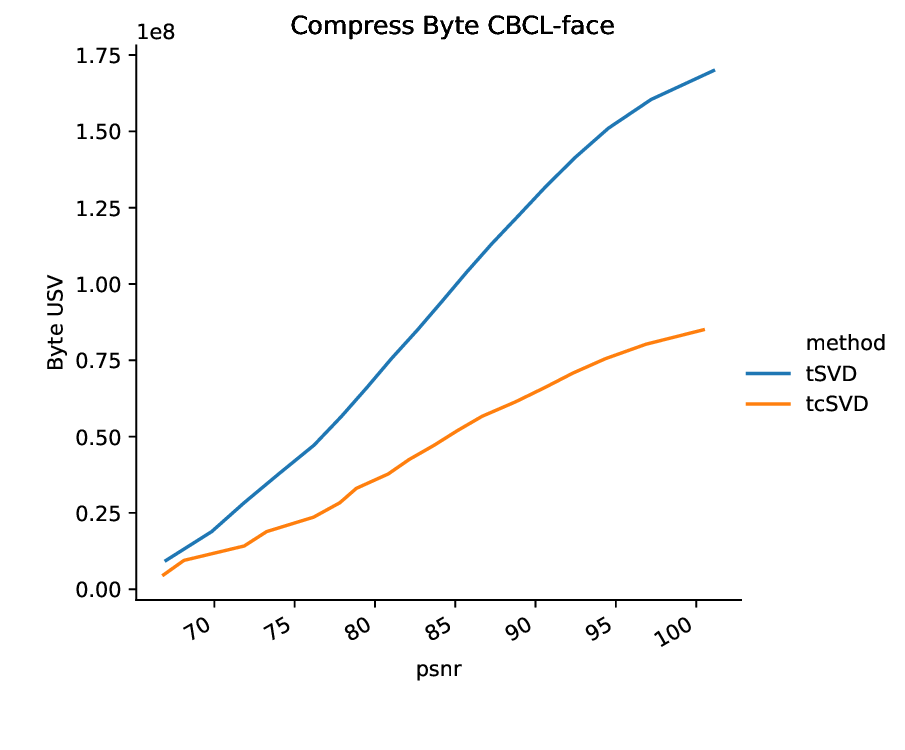}
	\end{subfigure}%
	\begin{subfigure}[t]{0.5\textwidth}
		\centering
		\includegraphics[width=0.65\textwidth]{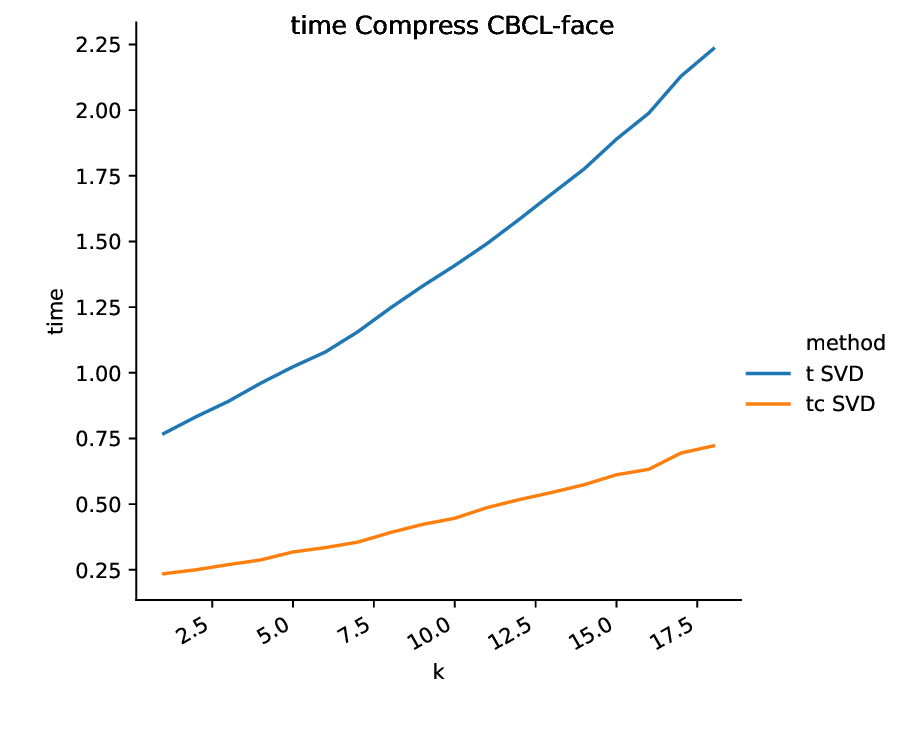}
	\end{subfigure}%
	\newline
	\begin{subfigure}[t]{0.5\textwidth}
		\centering
		\includegraphics[width=0.65\textwidth]{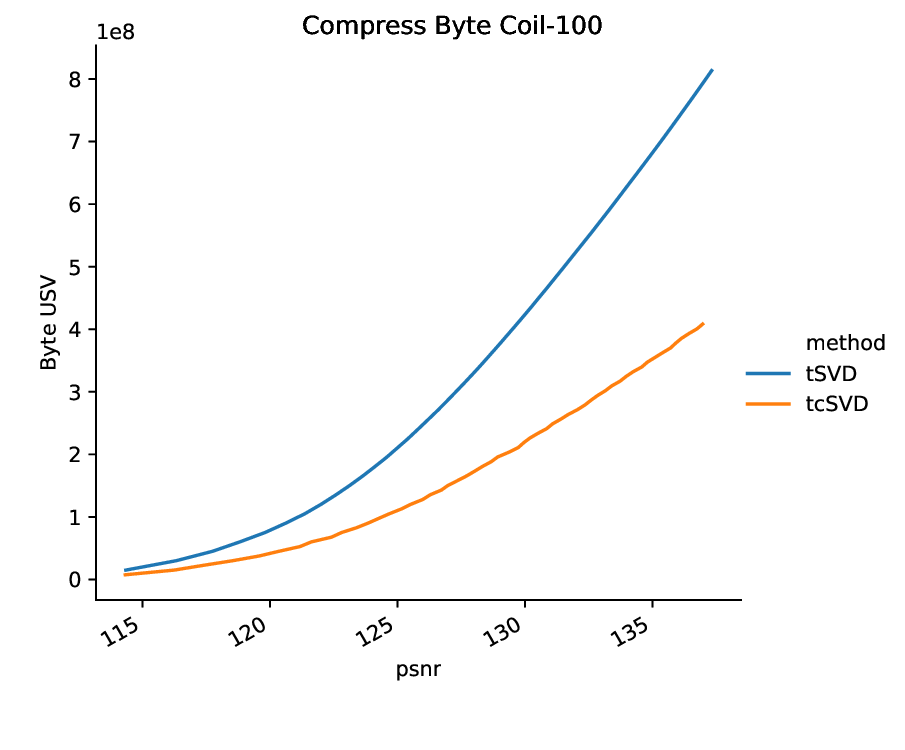}
	\end{subfigure}%
	\begin{subfigure}[t]{0.5\textwidth}
		\centering
		\includegraphics[width=0.65\textwidth]{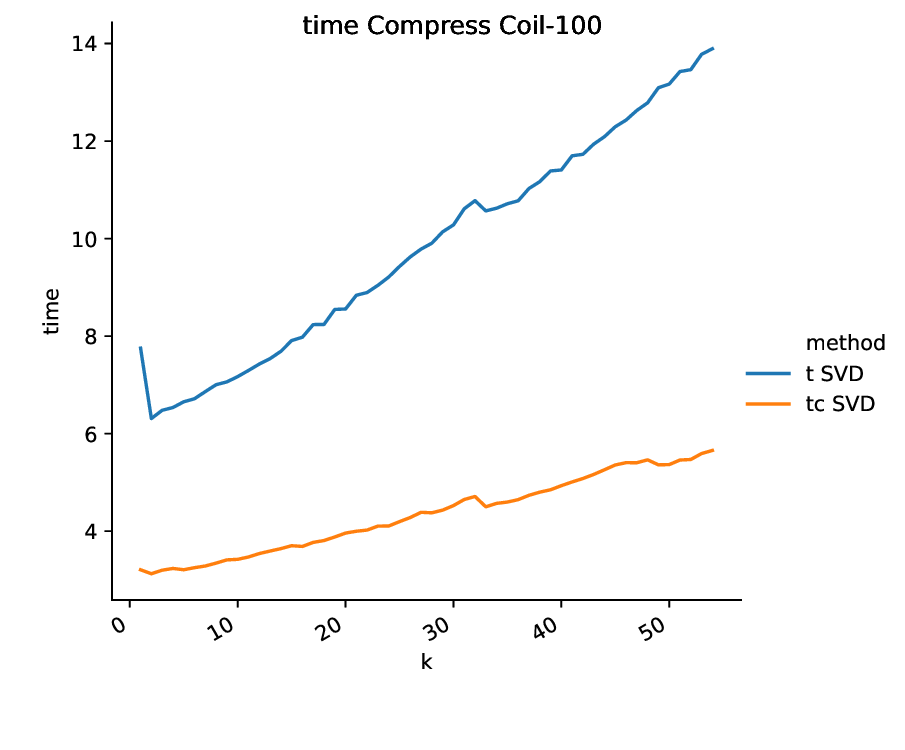}
	\end{subfigure}%
	\newline
	\begin{subfigure}[t]{0.5\textwidth}
		\centering
		\includegraphics[width=0.65\textwidth]{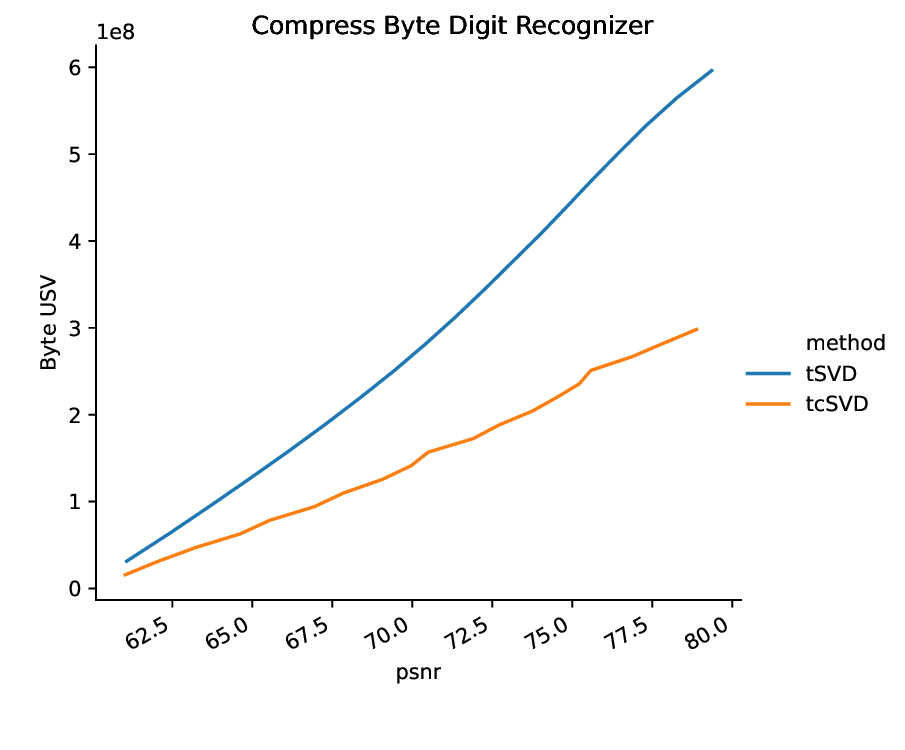}
	\end{subfigure}%
	\begin{subfigure}[t]{0.5\textwidth}
		\centering
		\includegraphics[width=0.65\textwidth]{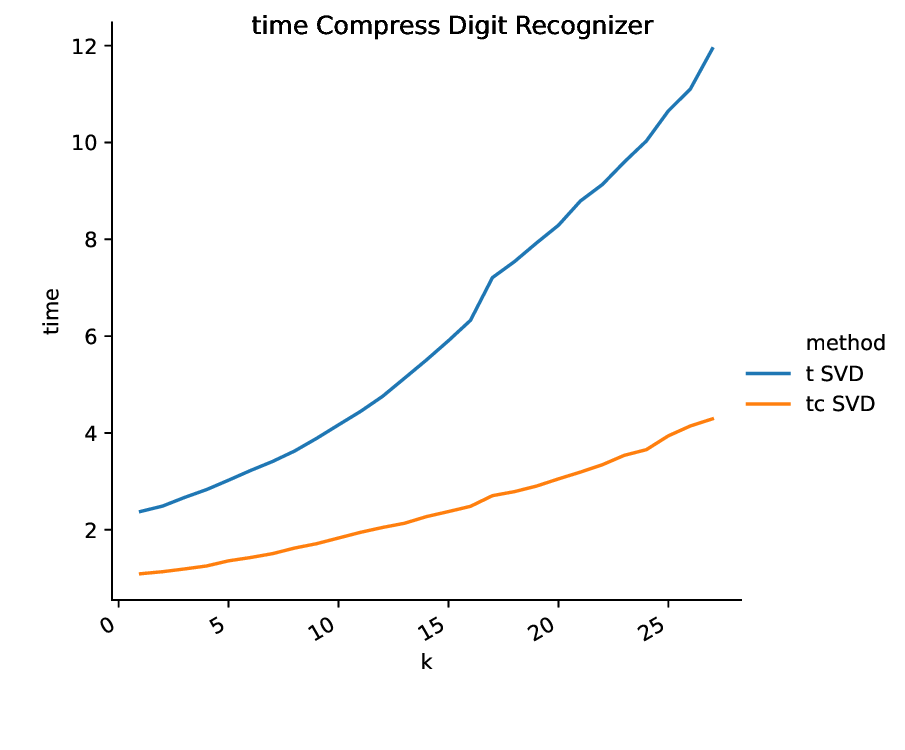}
	\end{subfigure}%
	\newline
	\begin{subfigure}[t]{0.5\textwidth}
		\centering
		\includegraphics[width=0.65\textwidth]{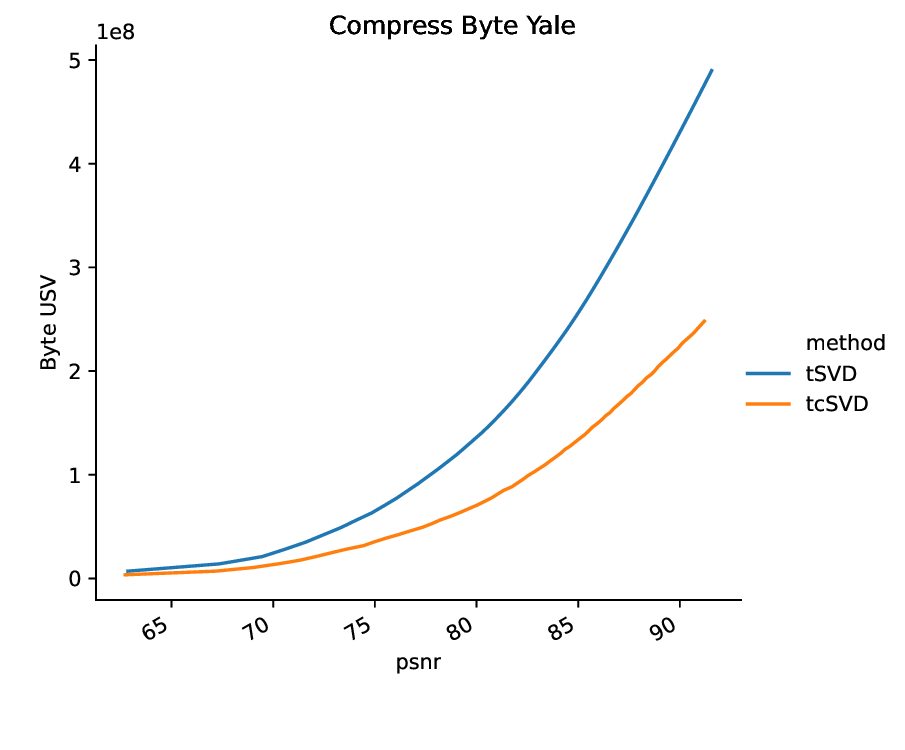}
	\end{subfigure}%
	\begin{subfigure}[t]{0.5\textwidth}
		\centering
		\includegraphics[width=0.65\textwidth]{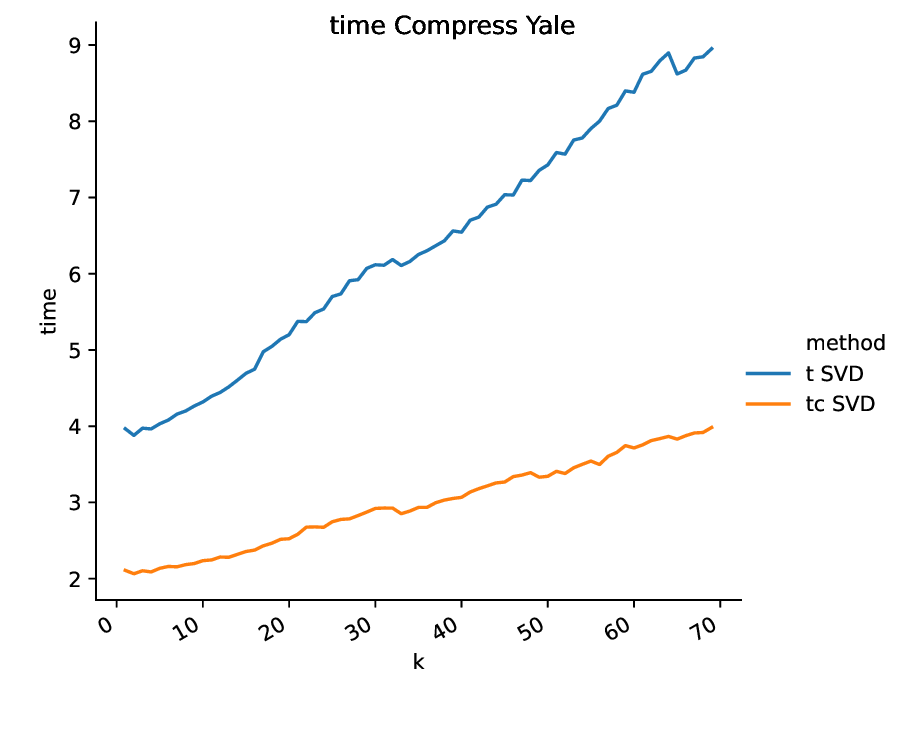}
	\end{subfigure}%
	
	\caption{(PSNR Vs storage) and ($k$ Vs Time) for $t$-SVD ans $\star_c$-SVD, on  some data from Table \ref{table_databaseName}}\label{figrun2}
\end{figure}

To perform classification, we used the truncated $\star_c{}\text{-SVD}$ method, as described in Algorithm \ref{AlgClassification}, where only $k$ columns of each frontal slice of $\mathcal{U}$ and $\mathcal{V}$ are selected \cite{Kilmer2}. 
\begin{algorithm}[h!]
    \caption{$\star_c{}\text{-Product}$ for classification method}\label{AlgClassification}
    \begin{algorithmic}[1]
        \Require  $\text{train} \in \mathbb{R}^{I_{1} \times I_{2} \times \cdots \times I_{N}} $,
        $\text{test} \in \mathbb{R}^{I_{1} \times \cdots \times I_N}$, $ k \in \mathbb{R}$,
        \Ensure label test
        \State $\mathcal{M} \leftarrow$ mean data
        \State $\mathcal{A} \leftarrow$ mean deviation from of train data
        \State $\mathcal{U}_k,\mathcal{S}_k,\mathcal{V}_k=\text{tc{-}SVD}(\mathcal{A},k)$
        \State $\mathcal{L}_k=\mathcal{U}^T_k\star_c \mathcal{A}$
        \State $\mathcal{T} \leftarrow \text{test data}-\mathcal{M}$
        \State $\mathcal{J}_k=\mathcal{U}^T_k\star_c \mathcal{T}$
        \State Calculate $\|\mathcal{L}_k- \mathcal{J}_k^{\bar{i}} \|_F$ for
        $\bar{i}=i_1,\cdots,i_N $
    \end{algorithmic}
\end{algorithm}
We report the accuracy of both t-SVD and $\star_c{}\text{-SVD}$ for various datasets and different values of $k$ in Table \ref{Taclass}. The running times of these methods on all datasets are shown in Figure \ref{figrun3}. Our method demonstrated significantly lower running times and better accuracy for lower values of $k$ than the t-SVD method, as indicated by Table \ref{Taclass} and Figure \ref{figrun3}.
\begin{figure}[ht]
    \centering
    \begin{subfigure}[t]{0.32\textwidth}
        \centering
        \includegraphics[width=0.95\textwidth]{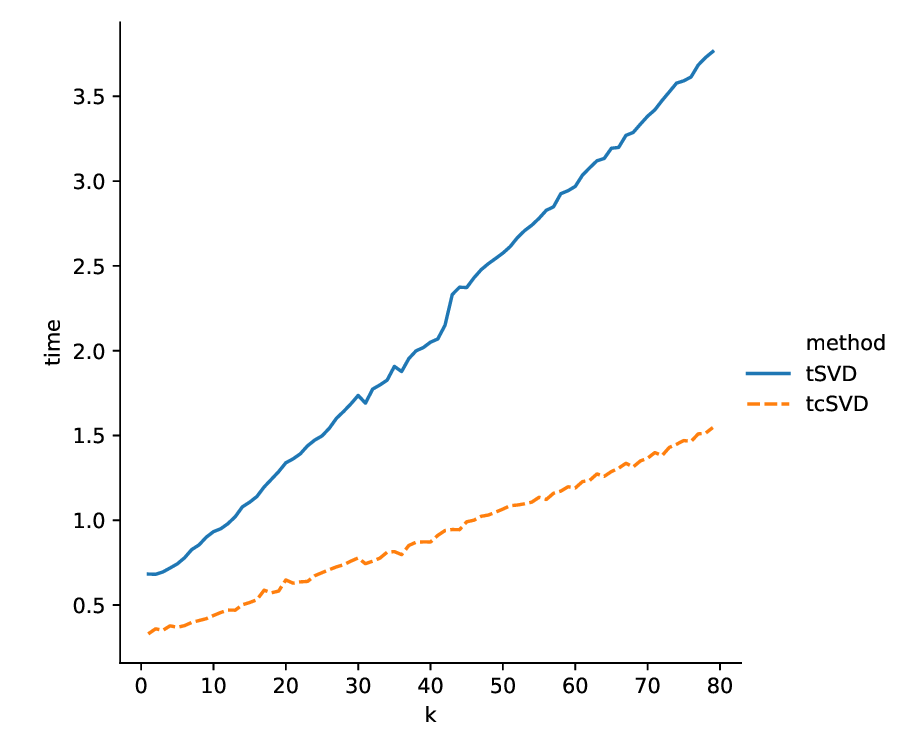}
        \caption{Brain MRI}
    \end{subfigure}%
    \begin{subfigure}[t]{0.32\textwidth}
        \centering
        \includegraphics[width=0.95\textwidth]{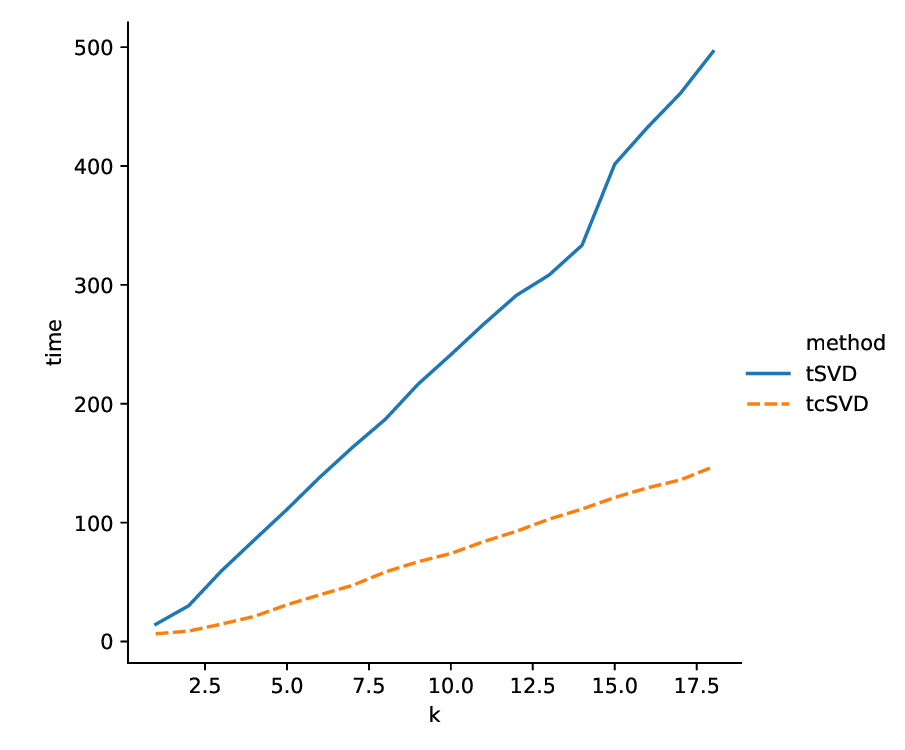}
        \caption{ CBCL-face}
    \end{subfigure}%
    \begin{subfigure}[t]{0.32\textwidth}
        \centering
        \includegraphics[width=0.95\textwidth]{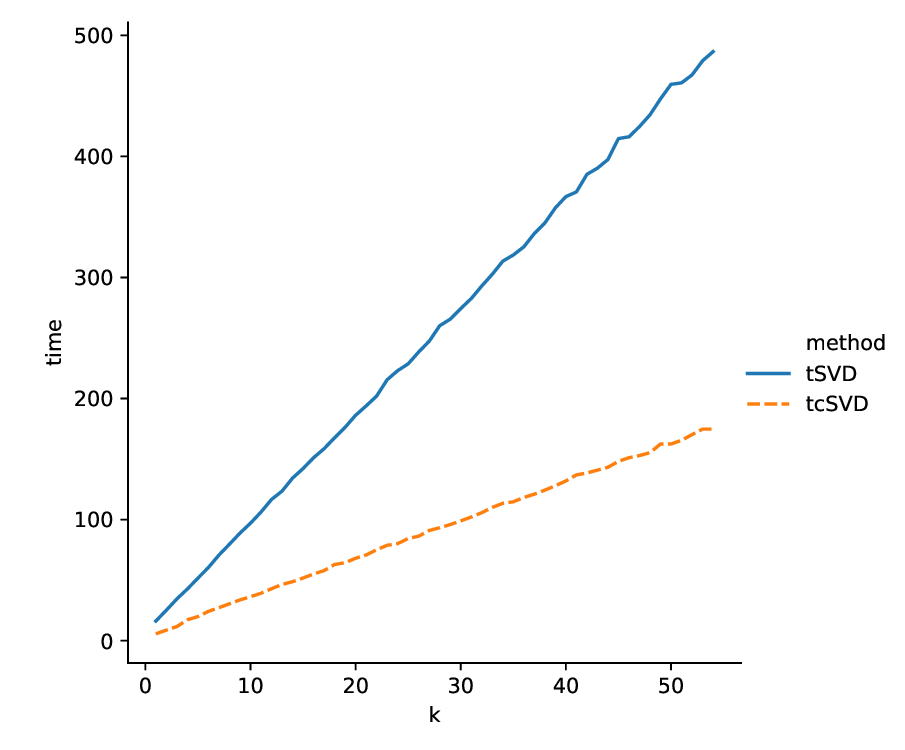}
        \caption{ Coil-100}
    \end{subfigure}%
    \newline
    \begin{subfigure}[t]{0.32\textwidth}
        \centering
        \includegraphics[width=0.95\textwidth]{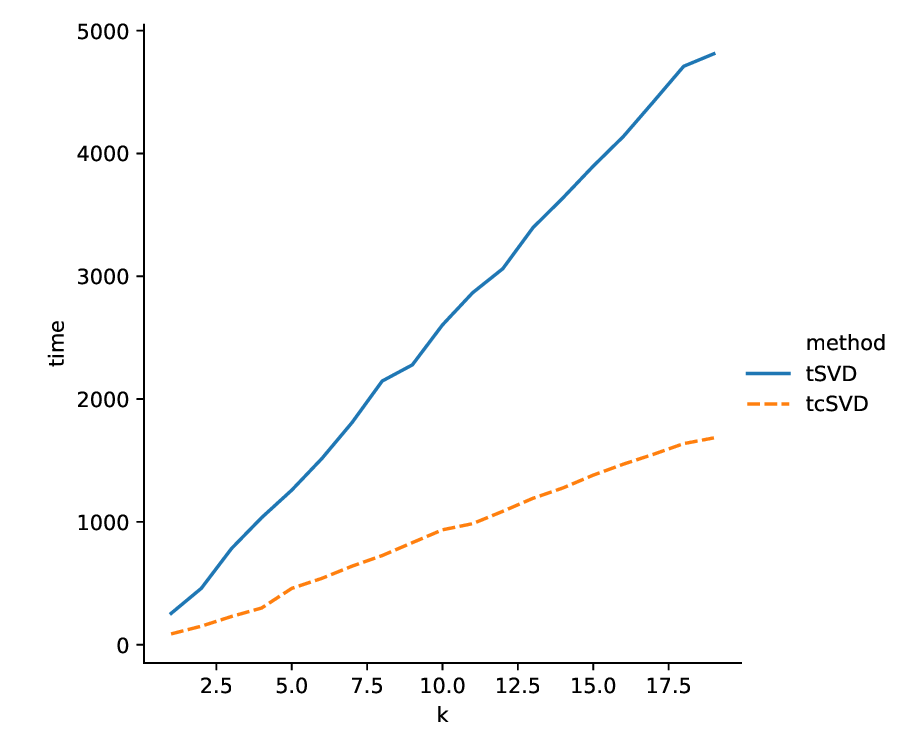}
        \caption{MNIST}
    \end{subfigure}%
    \begin{subfigure}[t]{0.32\textwidth}
        \centering
        \includegraphics[width=0.95\textwidth]{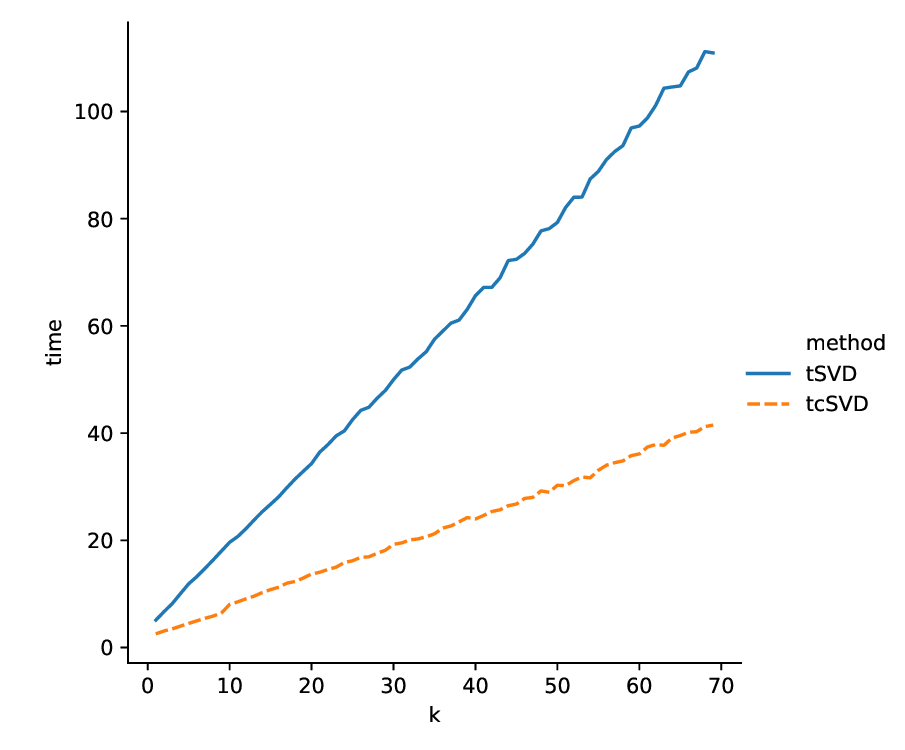}
        \caption{ Yale B}
    \end{subfigure}%

    \caption{The running time of classification based on $t$-SVD and $\star_c$-SVD methods according to differennt reductions number $k$.}\label{figrun3}
\end{figure}



\begin{table}[h]
    \caption{\footnotesize{Results of apply methods for clustering and classification tasks}}
    \begin{subtable}[t]{0.48\textwidth}
        \caption{\footnotesize{k-means method with best NMI}\label{Taclustek-means}}
        \resizebox{1\linewidth}{!}{
            \begin{tabular}[t]{l|lrr|cc}
                \toprule
                Dataset Name               & decompositin & k1 & k2 & NMI      & time     \\
                \midrule
                \multirow{2}{*}{Letters}   & tSVD         & 3  & 2  & 0.377755 & 1.65931  \\
                {}                         & tcSVD        & 3  & 3  & 0.398108 & 1.10069  \\ \cline{2-6}
                \multirow{2}{*}{MNIST}     & tSVD         & 5  & 21 & 0.416287 & 3.86368  \\
                {}                         & tcSVD        & 6  & 26 & 0.371089 & 2.06664  \\ \cline{2-6}
                \multirow{2}{*}{PIE\_pose} & tSVD         & 2  & 29 & 0.909975 & 0.208724 \\
                {}                         & tcSVD        & 22 & 20 & 0.913545 & 0.208813 \\ \cline{2-6}
                \multirow{2}{*}{PenDigits} & tSVD         & 3  & 3  & 0.627107 & 1.14224  \\
                {}                         & tcSVD        & 3  & 2  & 0.644992 & 0.825843 \\ \cline{2-6}
                \multirow{2}{*}{USPS}      & tSVD         & 5  & 13 & 0.370718 & 0.697890 \\
                {}                         & tcSVD        & 7  & 14 & 0.344233 & 0.343368 \\
                \bottomrule
            \end{tabular}
        }
    \end{subtable}
    \begin{subtable}[t]{0.48\textwidth}
        \caption{\footnotesize{Accuracy of classification on different dataset}\label{Taclass}}
        \resizebox{1\linewidth}{!}{
            \begin{tabular}{l|ll|cc}
                \toprule
                Data Set                   & method & k  & accuracy score & time        \\
                \midrule
                \multirow{2}{*}{Brain MRI} & tSVD   & 49 & 0.803922       & 2.544043    \\
                                           & tcSVD  & 62 & 0.803922       & 1.235369    \\  \cline{2-5}
                \multirow{2}{*}{CBCL-face} & tSVD   & 18 & 0.965357       & 496.210606  \\
                                           & tcSVD  & 18 & 0.965357       & 146.935816  \\ \cline{2-5}
                \multirow{2}{*}{Coil-100}  & tSVD   & 2  & 0.983333       & 25.173541   \\
                                           & tcSVD  & 2  & 0.984722       & 8.707815    \\ \cline{2-5}
                \multirow{2}{*}{MNIST}     & tSVD   & 8  & 0.972821       & 2146.879512 \\
                                           & tcSVD  & 9  & 0.972107       & 831.023561  \\ \cline{2-5}
                \multirow{2}{*}{Yale}      & tSVD   & 41 & 0.772257       & 67.171898   \\
                                           & tcSVD  & 47 & 0.772257       & 80.1        \\\cline{2-5}
                \multirow{2}{*}{Cifar}     & tSVD   & 9  & 0.62           & 67.171898   \\
                                           & tcSVD  & 8  & 0.63           & 28.30.598   \\\cline{2-5}
                \multirow{2}{*}{3D-MNIST}  & tSVD                                       \\
                                           & tcSVD                                      \\
                \bottomrule
            \end{tabular}}
    \end{subtable}
\end{table}

\section{Conclusion}
In this paper, we presented a new tensor-tensor product, $\star_c{}\text{-SVD}$, which is based on the block convolution operation between two tensors. We demonstrated that this product provides a general tensor decomposition method that extends the SVD decomposition to arbitrary tensors, and we showed that the coefficient matrix in our linear system can be diagonalized by the Discrete Cosine Transform (DCT), which is faster than the Fast Fourier Transform. Our experimental results on both synthetic random tensors and well-known datasets showed that our proposed method is faster and more accurate than the traditional t-SVD method. Our $\star_c{}\text{-SVD}$ product has potential applications in various fields such as image processing, signal processing, and machine learning. Future research can explore more efficient algorithms to compute the $\star_c{}\text{-SVD}$ decomposition and investigate its applications in other areas.

\vskip 0.2in
\bibliography{cas-refs}

\begin{thebibliography}{36}
\providecommand{\natexlab}[1]{#1}
\providecommand{\url}[1]{\texttt{#1}}
\expandafter\ifx\csname urlstyle\endcsname\relax
  \providecommand{\doi}[1]{doi: #1}\else
  \providecommand{\doi}{doi: \begingroup \urlstyle{rm}\Url}\fi

\bibitem[Ai et~al.(2018)Ai, Ma, Du, and Fang]{dMRI}
Jianhang Ai, Shuli Ma, Huiqian Du, and Liping Fang.
\newblock Dynamic mri reconstruction using tensor-svd.
\newblock In \emph{2018 14th IEEE International Conference on Signal Processing
  (ICSP)}, pages 1114--1118. IEEE, 2018.

\bibitem[Bader and Kolda(2006)]{alg862}
Brett~W Bader and Tamara~G Kolda.
\newblock Algorithm 862: Matlab tensor classes for fast algorithm prototyping.
\newblock \emph{ACM Transactions on Mathematical Software (TOMS)}, 32\penalty0
  (4):\penalty0 635--653, 2006.

\bibitem[Bahri et~al.(2018)Bahri, Panagakis, and Zafeiriou]{moden}
Mehdi Bahri, Yannis Panagakis, and Stefanos Zafeiriou.
\newblock Robust kronecker component analysis.
\newblock \emph{IEEE transactions on pattern analysis and machine
  intelligence}, 41\penalty0 (10):\penalty0 2365--2379, 2018.

\bibitem[Bi and Wang(2019)]{EEG}
Xiaojun Bi and Haibo Wang.
\newblock Early alzheimer’s disease diagnosis based on eeg spectral images
  using deep learning.
\newblock \emph{Neural Networks}, 114:\penalty0 119--135, 2019.

\bibitem[Cai et~al.(2006)Cai, He, Wen, Han, and Ma]{STM}
Deng Cai, Xiaofei He, Ji-Rong Wen, Jiawei Han, and Wei-Ying Ma.
\newblock Support tensor machines for text categorization.
\newblock Technical report, 2006.

\bibitem[Chakrabarty(2019)]{datasetBrainMRI}
Navoneel Chakrabarty.
\newblock {Brain MRI Images for Brain Tumor Detection}.
\newblock
  \url{https://www.kaggle.com/datasets/navoneel/brain-mri-images-for-brain-tumor-detection},
  2019.

\bibitem[Fan et~al.(2020)Fan, Ding, Yang, and Udell]{cp1}
Jicong Fan, Lijun Ding, Chengrun Yang, and Madeleine Udell.
\newblock Low-rank tensor recovery with euclidean-norm-induced schatten-p
  quasi-norm regularization.
\newblock \emph{arXiv preprint arXiv:2012.03436}, 2020.

\bibitem[Goodfellow et~al.(2016)Goodfellow, Bengio, and
  Courville]{GoodBengCour16}
Ian~J. Goodfellow, Yoshua Bengio, and Aaron Courville.
\newblock \emph{Deep Learning}.
\newblock MIT Press, Cambridge, MA, USA, 2016.
\newblock \url{http://www.deeplearningbook.org}.

\bibitem[Han and Kan(2019)]{blur2}
Yue Han and Jiangming Kan.
\newblock Blind color-image deblurring based on color image gradients.
\newblock \emph{Signal Processing}, 155:\penalty0 14--24, 2019.

\bibitem[Hao et~al.(2013)Hao, Kilmer, Braman, and Hoover]{Kilmer2}
Ning Hao, Misha~E Kilmer, Karen Braman, and Randy~C Hoover.
\newblock Facial recognition using tensor-tensor decompositions.
\newblock \emph{SIAM Journal on Imaging Sciences}, 6\penalty0 (1):\penalty0
  437--463, 2013.

\bibitem[Kilmer and Martin(2011)]{tSVD}
Misha~E Kilmer and Carla~D Martin.
\newblock Factorization strategies for third-order tensors.
\newblock \emph{Linear Algebra and its Applications}, 435\penalty0
  (3):\penalty0 641--658, 2011.

\bibitem[Kolda and Bader(2009)]{CPTuker}
Tamara~G Kolda and Brett~W Bader.
\newblock Tensor decompositions and applications.
\newblock \emph{SIAM review}, 51\penalty0 (3):\penalty0 455--500, 2009.

\bibitem[Liang(2007)]{ClSVD}
Faming Liang.
\newblock Use of svd-based probit transformation in clustering gene expression
  profiles.
\newblock \emph{Computational Statistics \& Data Analysis}, 51\penalty0
  (12):\penalty0 6355--6366, 2007.

\bibitem[Liu et~al.(2010)Liu, Liu, and Chan]{MLDA}
Yang Liu, Yan Liu, and Keith~CC Chan.
\newblock Tensor distance based multilinear locality-preserved maximum
  information embedding.
\newblock \emph{IEEE Transactions on neural networks}, 21\penalty0
  (11):\penalty0 1848--1854, 2010.

\bibitem[Lou and Cheung(2019)]{videorec}
Jian Lou and Yiu-Ming Cheung.
\newblock Robust low-rank tensor minimization via a new tensor spectral $ k
  $-support norm.
\newblock \emph{IEEE Transactions on Image Processing}, 29:\penalty0
  2314--2327, 2019.

\bibitem[Lu et~al.(2019)Lu, Feng, Chen, Liu, Lin, and Yan]{TRPCA}
Canyi Lu, Jiashi Feng, Yudong Chen, Wei Liu, Zhouchen Lin, and Shuicheng Yan.
\newblock Tensor robust principal component analysis with a new tensor nuclear
  norm.
\newblock \emph{IEEE transactions on pattern analysis and machine
  intelligence}, 42\penalty0 (4):\penalty0 925--938, 2019.

\bibitem[Lu et~al.(2008)Lu, Plataniotis, and Venetsanopoulos]{MPCA}
Haiping Lu, Konstantinos~N Plataniotis, and Anastasios~N Venetsanopoulos.
\newblock Mpca: Multilinear principal component analysis of tensor objects.
\newblock \emph{IEEE transactions on Neural Networks}, 19\penalty0
  (1):\penalty0 18--39, 2008.

\bibitem[Martin et~al.(2013)Martin, Shafer, and LaRue]{n-dim-kilmer}
Carla~D Martin, Richard Shafer, and Betsy LaRue.
\newblock An order-p tensor factorization with applications in imaging.
\newblock \emph{SIAM Journal on Scientific Computing}, 35\penalty0
  (1):\penalty0 A474--A490, 2013.

\bibitem[Mohan et~al.(2020)Mohan, Panas, and Cuadra]{blur1}
K~Aditya Mohan, Robert~M Panas, and Jefferson~A Cuadra.
\newblock Saber: A systems approach to blur estimation and reduction in x-ray
  imaging.
\newblock \emph{IEEE Transactions on Image Processing}, 29:\penalty0
  7751--7764, 2020.

\bibitem[Nene et~al.(1988)Nene, Nayar, Murase, et~al.]{nene1988columbia}
S~Nene, S~Nayar, Hiroshi Murase, et~al.
\newblock Columbia object image library (coil 100) 1996.
\newblock \emph{Columbia University}, 1\penalty0 (2):\penalty0 3, 1988.

\bibitem[Ng et~al.(1999)Ng, Chan, and Tang]{rboundary}
Michael~K Ng, Raymond~H Chan, and Wun-Cheung Tang.
\newblock A fast algorithm for deblurring models with neumann boundary
  conditions.
\newblock \emph{SIAM Journal on Scientific Computing}, 21\penalty0
  (3):\penalty0 851--866, 1999.

\bibitem[Noroozi and Rezghi(2020)]{rsfMRI}
Ali Noroozi and Mansoor Rezghi.
\newblock A tensor-based framework for rs-fmri classification and functional
  connectivity construction.
\newblock \emph{Frontiers in Neuroinformatics}, 14, 2020.

\bibitem[Rezghi(2017)]{moden1}
Mansoor Rezghi.
\newblock A novel fast tensor-based preconditioner for image restoration.
\newblock \emph{IEEE Transactions on Image Processing}, 26\penalty0
  (9):\penalty0 4499--4508, 2017.

\bibitem[Rezghi and Amirmazlaghani(2019)]{rezghi-toplitz}
Mansoor Rezghi and Maryam Amirmazlaghani.
\newblock Even-order toeplitz tensor: framework for multidimensional structured
  linear systems.
\newblock \emph{Computational and Applied Mathematics}, 38\penalty0
  (3):\penalty0 1--24, 2019.

\bibitem[Rezghi and Elden(2011)]{pboundary}
Mansoor Rezghi and Lars Elden.
\newblock Diagonalization of tensors with circulant structure.
\newblock \emph{Linear Algebra and its Applications}, 435\penalty0
  (3):\penalty0 422--447, 2011.

\bibitem[Rezghi et~al.(2014)Rezghi, Hosseini, and Elden]{zboundary}
Mansoor Rezghi, S~Mohammad Hosseini, and Lars Elden.
\newblock Best kronecker product approximation of the blurring operator in
  three dimensional image restoration problems.
\newblock \emph{SIAM Journal on Matrix Analysis and Applications}, 35\penalty0
  (3):\penalty0 1086--1104, 2014.

\bibitem[Selvan and Ramakrishnan(2007)]{CSVD}
Srinivasan Selvan and Srinivasan Ramakrishnan.
\newblock Svd-based modeling for image texture classification using wavelet
  transformation.
\newblock \emph{IEEE transactions on image processing}, 16\penalty0
  (11):\penalty0 2688--2696, 2007.

\bibitem[Shen and Sethi(1996)]{edge1}
Bo~Shen and Ishwar~K Sethi.
\newblock Convolution-based edge detection for image/video in block dct domain.
\newblock \emph{Journal of Visual Communication and Image Representation},
  7\penalty0 (4):\penalty0 411--423, 1996.

\bibitem[Wang et~al.(2020)Wang, Jin, and Tang]{tsvdapp4}
Andong Wang, Zhong Jin, and Guoqing Tang.
\newblock Robust tensor decomposition via t-svd: Near-optimal statistical
  guarantee and scalable algorithms.
\newblock \emph{Signal Processing}, 167:\penalty0 107319, 2020.

\bibitem[Wang and Zhu(2017)]{DSVD}
Yongchang Wang and Ligu Zhu.
\newblock Research and implementation of svd in machine learning.
\newblock In \emph{2017 IEEE/ACIS 16th International Conference on Computer and
  Information Science (ICIS)}, pages 471--475. IEEE, 2017.

\bibitem[Wen et~al.(2018)Wen, Liu, Ma, Jian, Lv, Hong, and Shi]{edge2}
Changbao Wen, Pengli Liu, Wenbo Ma, Zhirong Jian, Changheng Lv, Jitong Hong,
  and Xiaowen Shi.
\newblock Edge detection with feature re-extraction deep convolutional neural
  network.
\newblock \emph{Journal of Visual Communication and Image Representation},
  57:\penalty0 84--90, 2018.

\bibitem[Wu et~al.(2022)Wu, Shu, Xu, Chang, Chen, and Zheng]{wu2022robust}
Zhebin Wu, Lin Shu, Ziyue Xu, Yaomin Chang, Chuan Chen, and Zibin Zheng.
\newblock Robust tensor graph convolutional networks via t-svd based graph
  augmentation.
\newblock In \emph{Proceedings of the 28th ACM SIGKDD Conference on Knowledge
  Discovery and Data Mining}, pages 2090--2099, 2022.

\bibitem[Xie et~al.(2017)Xie, Zhao, Meng, and Xu]{cp2}
Qi~Xie, Qian Zhao, Deyu Meng, and Zongben Xu.
\newblock Kronecker-basis-representation based tensor sparsity and its
  applications to tensor recovery.
\newblock \emph{IEEE transactions on pattern analysis and machine
  intelligence}, 40\penalty0 (8):\penalty0 1888--1902, 2017.

\bibitem[Zaki et~al.(2014)Zaki, Meira~Jr, and Meira]{Zaki}
Mohammed~J Zaki, Wagner Meira~Jr, and Wagner Meira.
\newblock \emph{Data mining and analysis: fundamental concepts and algorithms}.
\newblock Cambridge University Press, 2014.

\bibitem[Zhang et~al.(2019)Zhang, Luo, Jiang, Yu, Xu, and Zhou]{im}
Fangyan Zhang, Ting Luo, Gangyi Jiang, Mei Yu, Haiyong Xu, and Wujie Zhou.
\newblock A novel robust color image watermarking method using rgb
  correlations.
\newblock \emph{Multimedia Tools and Applications}, 78\penalty0 (14):\penalty0
  20133--20155, 2019.

\bibitem[Zhang and Aeron(2016)]{tensorcompletion}
Zemin Zhang and Shuchin Aeron.
\newblock Exact tensor completion using t-svd.
\newblock \emph{IEEE Transactions on Signal Processing}, 65\penalty0
  (6):\penalty0 1511--1526, 2016.

\end{thebibliography}

\end{document}